\documentclass[10pt]{article} 
\usepackage[accepted]{tmlr}

\usepackage{amsmath,amsfonts,bm}

\def\eqref#1{equation~\ref{#1}}

\def\1{\bm{1}}

\DeclareMathAlphabet{\mathsfit}{\encodingdefault}{\sfdefault}{m}{sl}
\SetMathAlphabet{\mathsfit}{bold}{\encodingdefault}{\sfdefault}{bx}{n}

\newcommand{\E}{\mathbb{E}}

\usepackage{hyperref}
\usepackage{url}

\usepackage{amsthm}
\newcommand{\rmd}{\mathrm{d}}
\newtheorem{thm}{Theorem}
\newtheorem{assumption}{Assumption}
\newtheorem{remark}[thm]{Remark}
\newtheorem{prop}[thm]{Proposition}
\newtheorem{lem}[thm]{Lemma}
\newtheorem{cor}[thm]{Corollary}
\theoremstyle{definition}
\newtheorem{defn}[thm]{Definition}

\usepackage{caption}
\usepackage{bbm}
\usepackage{enumitem}
\usepackage{float}
\usepackage{booktabs}
\usepackage{subcaption}
\usepackage{graphicx}
\usepackage[titletoc,toc,title]{appendix}
\usepackage{makecell}
\newtheorem{intassumption}{Assumption}
\numberwithin{intassumption}{assumption}
\renewcommand{\theintassumption}{\theassumption.\alph{intassumption}}

\title{Wasserstein Convergence of Score-based Generative Models under Semiconvexity and Discontinuous Gradients}

\author{\name Stefano Bruno \email sbruno@ed.ac.uk \\
      \addr School of Mathematics \\
      University of Edinburgh, United Kingdom
      \\ Department of Industrial Engineering \\
      Ulsan National Institute of Science and Technology, Republic of Korea
      \AND
      \name Sotirios Sabanis \email s.sabanis@ed.ac.uk \\
      \addr School of Mathematics
      \\ University of Edinburgh, United Kingdom
      \\ National Technical University of Athens, Greece
      \\ Archimedes, Athena Research Center, Greece}

\def\month{09}  
\def\year{2025} 
\def\openreview{\url{https://openreview.net/forum?id=vS9iVRB7XF}} 

\begin{document}

\maketitle

\begin{abstract}
\noindent Score-based Generative Models (SGMs) approximate a data distribution by perturbing it with Gaussian noise and subsequently denoising it via a learned reverse diffusion process. These models excel at modeling complex data distributions and generating diverse samples, achieving state-of-the-art performance across domains such as computer vision, audio generation, reinforcement learning, and computational biology. Despite their empirical success, existing Wasserstein-2 convergence analysis typically assume strong regularity conditions--such as smoothness or strict log-concavity of the data distribution--that are rarely satisfied in practice.
	In this work, we establish the first non-asymptotic Wasserstein-2 convergence guarantees for SGMs targeting semiconvex distributions with potentially discontinuous gradients. Our upper bounds are explicit and sharp in key parameters, achieving optimal dependence of $O(\sqrt{d})$ on the data dimension $d$ and convergence rate of order one. The framework accommodates a wide class of practically relevant distributions, including symmetric modified half-normal distributions, Gaussian mixtures, double-well potentials, and elastic net potentials. By leveraging semiconvexity without requiring smoothness assumptions on the potential such as differentiability, our results substantially broaden the theoretical foundations of SGMs, bridging the gap between empirical success and rigorous guarantees in non-smooth, complex data regimes.
\end{abstract}

\section{Introduction}    \label{Introduction}

Score-based Generative Models (SGMs), also known as diffusion-based generative models \citep{NEURIPS2019_3001ef25,song2021scorebased,sohl2015deep,ho2020denoising}, have rapidly emerged over the past few years as a popular approach in modern generative modelling due to their remarkable capabilities in generating complex data, surpassing previous state-of-the-art models, such as Generative Adversarial Networks (GANs) \citep{goodfellow2014generative} and Variational AutoEncoders (VAEs) \citep{kingma2014auto}.  These models are now
widely adopted in computer vision and audio generation tasks \citep{kong2020diffwave,chen2020wavegrad,mittal2021symbolicdiffusion,Avrahami2021BlendedDF,Kim2021DiffusionCLIPTD,Bansal2023UniversalGF,Saharia2022PhotorealisticTD,Po2023StateOT,Zhang2023ASO}, text
generation \citep{Li2022DiffusionLMIC,Yu2022LatentDE,Lovelace2022LatentDF}, sequential data modeling \citep{LopezAlcaraz2022DiffusionbasedTS,Tashiro2021CSDICS,Tevet2022HumanMD}, reinforcement learning and control \citep{Pearce2023ImitatingHB,Chi2023DiffusionPV,HansenEstruch2023IDQLIQ,Reuss2023GoalConditionedIL,Zhu2023DiffusionMF,Ding2023ConsistencyMA}, as
well as life-science \citep{Chung2021ScorebasedDM,Jing2022TorsionalDF,Watson2023DeND,Song2021SolvingIP,Weiss2023GuidedDF}.  We refer the reader to the survey papers  \citet{yang2023diffusion, chen2024overview} for a more comprehensive exposition of their applications. 


\noindent The primary goal of SGMs is to generate synthetic data that closely match a target data distribution $\pi_{\mathsf{D}}$, given a sample set. In particular, these models generate approximate data samples from high-dimensional data distributions by combining two diffusion processes, a forward and a backward process in time. The forward process  is used to iteratively and smoothly transform samples from the unknown data distribution into (Gaussian) noise, while the associated backward process reverses the noising procedure to generate new samples from the starting unknown data distribution. A key role in these models is played by the score function, i.e. the gradient of the log-density of the solution of the forward process, which appears in the drift of the stochastic differential equation (SDE) associated with the backward process. Since this quantity depends on the unknown data distribution, an estimator of the score must be constructed during the noising step using score-matching techniques \citep{hyvarinen2005estimation,vincent2011connection}. 

\noindent The widespread applicability and success of SGMs have been accompanied by a growing interest in the theoretical understandings of these models, particularly in the convergence analysis under different metrics such as Total Variation (TV) distance, Kullback Leibler (KL) divergence, Wasserstein distance,  e.g.,  \citet{block2020generative, de2021diffusion, debortoli2022convergence, lee2022convergence,yang2022convergence, kwon2022score, liu2022let, oko2023diffusion, lee2023convergence,chen2023improved,chen2023sampling, li2023towards,  pedrotti2023improved, conforti2023score, benton2023linear, strasman2025an, bruno2025on, tang2024contractive, mimikos-stamatopoulos2024scorebased, wang2024wasserstein, gentiloni2025beyond,yu2025advancing}. In this work, we provide a non-asymptotic convergence analysis in Wasserstein distance of order two, as this metric is often considered more practical and informative for estimation tasks (see e.g.,  \eqref{eq:obj_explicit_score_matching_general_no_weighting}), and is closely connected to the popular Fr{\'e}chet Inception Distance (FID) used to assess the quality of images in generative modeling (see, e.g., Section \ref{sec:related_work}).  A significant limitation of prior analysis in Wasserstein-2, e.g.,  \citet{strasman2025an,gao2023wasserstein,bruno2025on,tang2024contractive,wang2024wasserstein,yu2025advancing}, is their reliance on strong regularity conditions--such as smoothness or strict log-concavity --  of the data distribution and its potential. These assumptions facilitate mathematical tractability but limit the applicability of theoretical results to more general settings, especially when the data distribution is only semiconvex and the potential’s gradient may be discontinuous. The only exception outside the strict log-concavity regime is the recent contribution in \citet{gentiloni2025beyond}, where the authors assumes that the data distribution is weakly convex. However, their analysis still requires the potential to be twice continuously differentiable (see, e.g., \citet[Proofs of Propositions B.1 and B.2]{gentiloni2025beyond}), and the stepsize of their generative algorithm must be bounded by a quantity inversely proportional to the one-sided Lipschitz constant of the potential (see \citet[equation (30)]{gentiloni2025beyond}). Still, such conditions on $\pi_{\mathsf{D}}$ in existing  Wasserstein-2 convergence analysis do not fully reflect the complexity of real-world data, which often exhibit non-smooth or non-log-concave distributions. Therefore, the aim of this work is to address the following fundamental question:
\begin{center} 
	\textit{Can Score-based Generative Models be guaranteed to converge in Wasserstein-2 distance when the data distribution is only semiconvex and the potential admits discontinuous gradients?}  \end{center}
\noindent We provide a positive answer to this question by combining recent findings in non-smooth, non-log-concave sampling, with standard stochastic analysis tools, thereby presenting the first contributions in the Score-based generative modeling literature for non-smooth potentials. We establish explicit, non-asymptotic  Wasserstein-2 convergence bounds for SGMs under semiconvexity assumptions on the data distribution,  accommodating potentials with discontinuous gradients.  This framework covers a variety of practically relevant distributions arising in Bayesian statistical methods, including symmetric modified half-normal distributions, Gaussian mixtures, double-well potentials, and elastic net potentials, all of which satisfy our relaxed assumptions. 

\noindent  In addition, our estimates are explicit and exhibit the best known optimal dependencies in terms of data dimension, i.e., $O(\sqrt{d})$ in Theorem \ref{main_theorem_general},  and rate of convergence, i.e., $O(\gamma)$ in Theorem \ref{main_theorem_general_better_rate_of_convergence}. In contrast to prior works under the same metric \citet{gentiloni2025beyond, gao2023wasserstein, strasman2025an, tang2024contractive}, our estimates in Theorem \ref{main_theorem_general} and Theorem  \ref{main_theorem_general_better_rate_of_convergence} are derived without imposing any restrictions on the stepsize of the generative algorithm\footnote{The results in \citet{gentiloni2025beyond, gao2023wasserstein, strasman2025an, tang2024contractive} require the stepsize to be controlled in terms of the Lipschitz constant or the strong convexity constant of the target data distribution; see, e.g.,  Table \ref{Table_comparison_stepsize} below. The only exceptions are the results in \citet[Remark 12 and Theorem 10]{bruno2025on}, which, however, requires stronger assumptions on the data distribution than Assumption \ref{assumption_score_strong_monotonicity_new_semi_convexity} below.}, making them more suitable for practical implementation. By circumventing the need for strict regularity conditions on the score function and allowing discontinuities in the gradients of the potentials, our work significantly expands the theoretical foundation of SGMs. This advancement not only bridges the gap between empirical success and theoretical guarantees but also opens new avenues for the application of diffusion models to data distributions with non-smooth potentials. 

\noindent One source of error in the construction of the generative algorithm arises from replacing the initial condition of the backward process with the invariant measure of the forward process. To ensure this error remains small, the drift terms of both SDEs must satisfy, for instance, a monotonicity property with a time-dependent bound that meets an appropriate integrability condition (see, e.g., \eqref{contraction_score_assumption_constant_semiconvexity} and \eqref{expression_constant_A_t} below). To address this, we identify a time horizon for the generative algorithm that ensures the paths of the two backward processes become contractive. Notably, the integrability condition on the monotonicity bound depends only on the known constants in Assumption \ref{assumption_score_strong_monotonicity_new_semi_convexity}, making it significantly easier to verify in practice compared to the analogous condition in \citet[Appendix C]{gentiloni2025beyond}, which relies on weak convexity constants that are often difficult to estimate. 


\noindent  In conclusion, we present the first explicit, dimension- and parameter-dependent $W_2$-convergence guarantees for Score-based Generative models operating on  data distributions having potentials with discontinuous gradients. Our results mark an important step forward in the rigorous analysis of SGMs, providing both theoretical insights and practical tools for advancing generative modeling in challenging, non-smooth regimes.

\textit{Notation.} Let $(\Omega, \mathcal{F}, \mathbb{P})$  be a fixed
probability space. We denote by $\mathbb{E}[X]$ the expectation of a random variable $X$. For $ 1 \le p < \infty$, $L^p$ is used to denote the usual space of $p$-integrable real-valued random
variables. The $L^p$-integrability of a random variable $X$ is defined as   $\mathbb{E}[|X|^p] < \infty$. Fix an integer $d \ge 1$. For an $\mathbb{R}^{d }$-valued random variable $X$, its law on $\mathcal{B}(\mathbb{R}^{d} )$, i.e. the
Borel sigma-algebra of $\mathbb{R}^{d} $  is denoted by $\mathcal{L}(X)$.   Let $T>0$ denote a time horizon. For a positive real number $b$, we denote its integer part by $\lfloor b \rfloor$. The Euclidean scalar product is denoted by $\langle \cdot, \cdot \rangle $, with $| \cdot |$ standing for the corresponding norm (where the dimension of the space may vary depending on the context). Let $f: \mathbb{R}^{d } \rightarrow \mathbb{R}$ be a continuously differentiable function. The gradient of $f$ is denote by $\nabla f$. For any integer $q \ge 1$, let $\mathcal{P}(\mathbb{R}^q)$ be the set of probability measures on $\mathcal{B}(\mathbb{R}^q)$. For $\mu$, $\nu \in \mathcal{P}(\mathbb{R}^d)$, let $\mathcal{C}(\mu, \nu)$ denote the set of probability measures $\zeta$ on $\mathcal{B}(\mathbb{R}^{2d })$ such that its respective marginals are $\mu$ and $\nu$. For any $\mu$ and $\nu \in \mathcal{P}(\mathbb{R}^{d })$, the Wasserstein distance of order 2 is defined as
\begin{equation*}
	W_2(\mu, \nu)= \left(  \inf_{\zeta \in \mathcal{C}(\mu, \nu)} \int_{\mathbb{R}^d}  \int_{\mathbb{R}^d}  |x - y|^2 \ \text{d} \zeta(x,y) \right)^{\frac{1}{2}}.
\end{equation*}
Table \ref{Table_notation} (Appendix \ref{Table_of_constants_appendix})  lists the main symbols used throughout this work along with references to where they are defined.

	\section{Technical Background for OU-based SGMs} \label{Technical_Background_section}


\noindent	In this section, we briefly summarize the construction of score-based generative models (SGMs) via diffusion processes, as introduced by \citet{song2021scorebased}. The core idea behind SGMs is to employ an ergodic (forward) diffusion process that gradually transforms the unknown data distribution $\pi_{\mathsf{D}} \in \mathcal{P}(\mathbb{R}^{d})$ into a known prior distribution. A backward (in time) process is then learned to transform the prior  back to the target distribution $\pi_{\mathsf{D}}$ by estimating the score function of the forward process. In our analysis, we consider the forward process $(X_t)_{t \in [0,T]}$  to be an Ornstein-Uhlenbeck (OU) process, consistent with the choice in the original paper \citet{song2021scorebased}
\begin{equation} \label{OU_process_introduction}
	\text{d} X_t = -  X_t \ \text{d}t + \sqrt{2  } \ \text{d}B_t, \quad X_0 \sim \pi_{\mathsf{D}},
\end{equation}
where $(B_t)_{t \in [0,T]}$ is an $d$-dimensional Brownian motion and we assume that  $\mathbb{E}[|X_0|^2] < \infty$. 

For target data distributions $\pi_{\mathsf{D}}$ that are absolutely continuous with respect to the Lebesgue measure, and whose densities are continuous and integrable,
the backward process $(Y_t)_{t \in [0,T]} = (X_{T-t})_{t \in [0,T]}$ is well defined \footnote{The regularity of the Ornstein--Uhlenbeck semigroup for all $t \in (0,T]$ (see Appendix~\ref{regularity_appendix_section}, and \citet[Proof of Proposition 3.1]{conforti2023score}) ensures that the necessary and sufficient conditions for the reversibility of the diffusion process are satisfied; see, e.g., \citet[Theorem 2.2]{millet1989integration}  or \cite[Theorem 2.1]{haussmann1986time}. These conditions on $\pi_{\mathsf{D}}$  are included in Assumption~\ref{assumption_score_strong_monotonicity_new_semi_convexity}.} \citep{millet1989integration, haussmann1986time},  and is given by
\begin{equation} \label{eq:backwardproc_real_initial_condition_introduction}
	\text{d} Y_t  =  (Y_t +2  \nabla \log p_{T-t}(Y_t)) \ \text{d} t +\sqrt{2  } \ \text{d} \bar{B}_t, \quad Y_0 \sim \mathcal{L}(X_T),
\end{equation}
where $\{p_t \}_{t \in [0,T]}$ is the family of densities of $\{ \mathcal{L}(X_t)  \}_{t \in (0,T]}$ with respect to the Lebesgue measure, $\bar{B}_t$ is an another Brownian motion independent of $B_t$ in \ref{OU_process_introduction} defined on $(\Omega, \mathcal{F}, \mathbb{P})$.
In practice, however, the initial distribution is taken to be the invariant measure of the forward process, which corresponds to the standard 
Gaussian distribution. As a result, the backward process in \ref{eq:backwardproc_real_initial_condition_introduction} becomes
\begin{equation} \label{Y_hat_auxiliary}
	\text{d} \widetilde{Y}_t =  (\widetilde{Y}_t + 2 \  \nabla \log p_{T-t} ( \widetilde{Y}_t ) ) \ \text{d} t + \sqrt{2  } \ \text{d} \bar{B}_t, \quad \widetilde{Y}_0 \sim \pi_{\infty} = \mathcal{N}(0,I_{d}).
\end{equation}
Since the target distribution $\pi_{\mathsf{D}}$ is unknown, the score function $ \nabla \log p_t$  in \ref{eq:backwardproc_real_initial_condition_introduction}  cannot be computed exactly. To overcome this limitation, an estimator $s(\cdot, \theta^{*}, \cdot)$  is \textit{learned} based on a family of functions $ s: [0,T] \times \mathbb{R}^M \times \mathbb{R}^d \rightarrow  \mathbb{R}^d $ parametrized in $\theta$, aiming at approximating the score of the ergodic forward process \ref{eq:obj_explicit_score_matching_general_no_weighting}  over a fixed time window $[0,T]$. In practice,  $s$ are neural networks and in particular cases, e.g., the motivating example in \citet[Section 3.1]{bruno2025on}, the functions $s$ can be carefully designed.
The optimal value  $\theta^*$ of the parameter $\theta$ is determined  by  optimizing the following score-matching objective
\begin{equation}\label{eq:obj_explicit_score_matching_general_no_weighting}
	\begin{split}
		\mathbb{R}^d \ni \theta \mapsto  \E \left[  \int_{0}^T  | \nabla \log p_{t}(X_t)  -  s(t, \theta , X_t) |^2 \  \rmd t \right].
	\end{split}
\end{equation}	
An explicit expression of the stochastic gradient of \ref{eq:obj_explicit_score_matching_general_no_weighting} derived via denoising score matching \citep{vincent2011connection} is provided in \citet[equation (8), Section 2]{bruno2025on}.	
\noindent  Following \citet[Section 2]{bruno2025on}, we define an auxiliary process $ (Y_t^{\text{aux}})_{t \in [0,T]}$ that incorporates the approximating function $s$, which depends on the (random) estimator  of $\theta^*$ denoted by $\hat{\theta} $.  For $t \in [0,T]$, this process is given by
\begin{equation}  \label{Y_auxiliary_theta_hat}
	\text{d} Y_t^{\text{aux}} =  (Y_t^{\text{aux}} + 2 \ s(T-t, \hat{\theta}, Y_t^{\text{aux}}) )  \ \text{d} t + \sqrt{2 } \ \text{d} \bar{B}_t, \quad Y_0^{\text{aux}} \sim \pi_{\infty} = \mathcal{N}(0,I_{d}).
\end{equation}
The auxiliary process \ref{Y_auxiliary_theta_hat} serves as a bridge between the
 backward process \ref{Y_hat_auxiliary} and the numerical scheme \ref{continuous_time_EM_version}, and it facilitates the analysis of the convergence of the diffusion model (see the upper bounds involving $Y_t^{\text{aux}} $ in the proof of Theorem \ref{main_theorem_general} in Appendix \ref{proof_of_main_Theorem_appendix_section}  for further details). 
We now introduce the numerical scheme. Let the step size $\gamma_{j}=\gamma \in (0,1)$ for each $j=0,\dots, J$, where $J \in \mathbb{N}$.
The discrete process $(Y_{j}^{\text{EM}})_{j \in \{0, \dots, J +1 \} }$ of the Euler--Maruyama approximation of \ref{Y_auxiliary_theta_hat} is given, for any $j \in \{0, \dots, J \}$, as follows
\begin{equation} \label{eq:backwardprocemdisc_old}
	Y_{j+1}^{\text{EM}} =	Y_{j}^{\text{EM}} +  \gamma (Y_{j}^{\text{EM}} + 2  \ s(T-t_{j}, \hat{\theta} ,  Y_{j}^{\text{EM}})) + \sqrt{2  \gamma   } \ \bar{Z}_{j+1}, \quad Y_{0}^{\text{EM}} \sim  \pi_{\infty} = \mathcal{N}(0,I_{d}),
\end{equation}
where $\{\bar{Z}_{j} \}_{j \in \{0, \dots, J +1 \}}$ is a sequence of independent $d$-dimensional Gaussian random variables with zero mean and identity covariance matrix. The continuous-time interpolation of \ref{eq:backwardprocemdisc_old},  for $t \in [0,T]$, is given by
\begin{equation} \label{continuous_time_EM_version}
	\rmd \widehat{Y}_t^{\text{EM}} =  (\widehat{Y}_{\lfloor t/\gamma \rfloor \gamma }^{\text{EM}} + 2 \ s(T-\lfloor t/\gamma \rfloor \gamma, \hat{\theta} , \widehat{Y}_{\lfloor t/\gamma \rfloor \gamma}^{\text{EM}})) \ \text{d}t + \sqrt{2 } \ \text{d}\bar{B}_t, \qquad \widehat{Y}_0^{\text{EM}} \sim \pi_{\infty} = \mathcal{N}(0,I_{d}),
\end{equation}
where $\mathcal{L}(\widehat{Y}_{j}^{\text{EM}}) = \mathcal{L}(Y^{\text{EM}}_{j})$ at grid points for each $j \in\{0, \dots, J +1\}$.

	\section{Wasserstein Convergence Analysis for SGMs} \label{Theoretical_convergence_section}
\noindent In this section, we provide the full non-asymptotic estimates in Wasserstein distance of order two between the target data distribution $\pi_{\mathsf{D}}$  and the generative distribution of the diffusion model under the assumptions stated below. As discussed in \citet[Section 2 and Appendix A]{bruno2025on}, it may be necessary to restrict  $t \in [\epsilon,T]$ for $\epsilon \in (0,1)$ in \ref{eq:obj_explicit_score_matching_general_no_weighting} to account for numerical instabilities that can arise during training and sampling near $t=0$ as also observed in practice in \citet[Appendix C]{song2021scorebased}, and for the possibility that the integral of the score function in \ref{eq:obj_explicit_score_matching_general_no_weighting} may diverge when $t= 0$. Therefore, we truncate the integration in the backward diffusion at $ T- \epsilon$ and consider the process $(Y_t)_{t \in [0,T - \epsilon]}$.




\subsection{Assumptions} \label{assumptions_section}
\noindent We begin by stating the main assumptions of our setting. The optimization problem in \ref{eq:obj_explicit_score_matching_general_no_weighting} can be solved using algorithms such as  stochastic gradient descent \citep{jentzen2021strong}, ADAM \citep{KingBa15}, Stochastic Gradient Langevin Dynamics \citep[Section 3.1]{bruno2025on}, and TheoPouLa \citep{SabanisLimTheoPoula}, provided they satisfy the following assumption.

\begin{assumption} \label{general_assumption_algorithm}
	Let $\theta^*$ be a minimiser\footnote{The score-matching optimization problem \ref{eq:obj_explicit_score_matching_general_no_weighting} is not necessarily (strongly) convex.} of \ref{eq:obj_explicit_score_matching_general_no_weighting} and let  $\hat{\theta}$ be the (random) estimator of $\theta^*$ obtained through some approximation procedure such that $	\mathbb{E} [ | \hat{\theta} |^2] < \infty$. There exists $ \widetilde{\varepsilon}_{\text{AL}} >0$  such that
	\begin{equation*}
		\mathbb{E} [ | \hat{\theta} - \theta^{*}|^2 ] <  \widetilde{\varepsilon}_{\text{AL}}.
	\end{equation*}
\end{assumption}
\begin{remark} \label{Control_algorithm}
	As a consequence of Assumption \ref{general_assumption_algorithm}, one obtains $	\mathbb{E} [ | \hat{\theta} |^2 ] < 2  \widetilde{\varepsilon}_{\text{AL}} + 2 | \theta^{*}|^2$.
\end{remark}

\noindent In this work, we consider the potentials to be semiconvex functions-- a broad generalization of convex functions that includes non-convex functions whose curvature is bounded from below. This class allows for discontinuities in the gradient while retaining key analytical properties of convex functions, such as the existence of well-defined subgradients.
\begin{defn}
    A function $U$ is semiconvex if there exists $K \ge 0$ such that $U+ \frac{K}{2} | \cdot |^2$ is convex.
\end{defn}
Semiconvexity  has received significant attention\footnote{Some of these references refer to semiconvex functions as weakly convex. We avoid this terminology to prevent confusion with the notion of weak convexity introduced in Definition~\ref{weak_convexity_definition_gentiloni} below.} in the machine  learning community \citep{davis2018subgradient, sun2019least,  richards2021learning, liu2021first, rafique2022weakly}, optimization \citep{li2019incremental, ma2020quadratically,li2021augmented,hu2025distributed}, optimal control \citep{cannarsa2004semiconcave},  and the study of fully nonlinear partial differential equations \citep{braga2019optimal,payne2023primer}.  We refer the reader to \citet{duda2009semiconvex, cattiaux2014semi} for comprehensive overviews of the mathematical challenges associated with this class.
Importantly, semiconvex functions may admit  discontinuous gradients which are characterized using the   Fr{\'e}chet subdifferential \citep{bazaraa1974cones, alberti1992singularities}.
\begin{defn} \label{definition_subdifferential}
    For a function $U:  \mathbb{R}^d \rightarrow \mathbb{R}$, we define the subdifferential $\partial U(x)$ of $U$ at $x \in \mathbb{R}^d $ as
    \begin{equation} \label{subdifferential_definition_potential}
	\partial U (x) = \left \{  \tilde{p} \in \mathbb{R}^d : \ \liminf_{z \rightarrow x} \frac{U(z)-U(x) - \langle \tilde{p}, z-x \rangle}{|z-x|} \ge 0 \right \}.
\end{equation}
\end{defn}
The set \ref{subdifferential_definition_potential} is closed and convex, and may be empty in general.  We say that $U$ is Fr{\'e}chet 
subdifferentiable at $x$ if  $\partial U(x) \neq \emptyset $. Any element $h(x) \in \partial U(x)$ is called a Fr{\'e}chet subgradient of $U$ at $x \in \mathbb{R}^d$. 
When $U$ is differentiable at $x$, the subdifferential reduces to $\partial U(x) = \{ \nabla U(x) \}$. 
Crucially, semiconvex functions are Fr{\'e}chet 
 subdifferentiable at every points in $\mathbb{R}^d $, i.e. $\partial U(x) \neq \emptyset$ for all $x $ \citep{alberti1992singularities,cannarsa2004semiconcave}. In this case, $h(x) + K x$ corresponds to the classical convex  subgradient \cite[Proposition 4.6]{vial1983strong}.   Moreover, every element of the subdifferential of a semiconvex function satisfies a one-sided Lipschitz condition, ensuring  the existence of $h(x) \in \partial U(x)$. The following lemma-- adapted from \citep[Proposition 2.1]{alberti1992singularities} and presented in \citep[Lemma 1]{johnston2025performance}--and its subsequent corollary formalize this property. 
\begin{lem}
\label{lemma_one_sided_Lipschitz_subdifferential} 
\citep[Modification of Proposition 2.1]{alberti1992singularities}
Let $U$ be a semiconvex function. Then, $U$ is locally Lipschitz continuous, the subdifferential set $\partial U$ is non-empty, compact, and $\tilde{p} \in \partial U(x)$, if and only if
\begin{equation*}
    U(z) - U(x) - \langle \tilde{p} , z - x \rangle \ge - \frac{K}{2} |z-x|^2,
\end{equation*}
for all $x,z \in \mathbb{R}^d$.
\end{lem}
\begin{cor}\citep[Corollary 1]{johnston2025performance}
Let $x, z \in \mathbb{R}^d$, $\tilde{p} \in \partial U(x)$, and $\tilde{q} \in \partial U(z)$. Then,
\begin{equation*}
    \langle \tilde{p} - \tilde{q} , x-z \rangle \ge -K | x-z |^2.
\end{equation*}
\end{cor}

We state the assumption on the target data distribution $\pi_{\mathsf{D}}$ below. Recall that $h(x) \in \partial U(x)$ is the Fr{\'e}chet subgradient of $U$ at $x \in \mathbb{R}^d$.  
\begin{assumption} \label{assumption_score_strong_monotonicity_new_semi_convexity}
	The data distribution $\pi_{\mathsf{D}}$  has a finite second moment and it is absolutely continuous with respect to the Lebesgue measure with  $\pi_{\mathsf{D}} (\rmd x) = \exp(-U(x)) \ \rmd x$ for some $U: \mathbb{R}^d \rightarrow \mathbb{R}$. Moreover,
	\begin{enumerate}[label=(\roman*)]
		\item The potential $U$ is continuous and its gradient exists almost everywhere.
		\item The potential $U$ is $K$-semiconvex (on a ball). That is, there exists $K, R \ge 0$, such that for all $ x, \bar{x} \in \mathbb{R}^d$,
		\begin{equation*}
			\langle  h(x) - h(\bar{x}), x - \bar{x}  \rangle  \ge -K  |x - \bar{x}|^2, \qquad \text{when} \quad |x - \bar{x}| < R,
		\end{equation*}
		\item The potential $U$ is $\mu$-strongly convex at infinity\footnote{Intuitively, outside a sufficiently large region, 
$U$ bends upwards at least as much as a quadratic function.}. That is, there exists $
		\mu>0$ such that for all $ x, \bar{x} \in \mathbb{R}^d$,
		\begin{equation} \label{contractivity_at_infinity_new}
			\langle h(x) -  h(\bar{x}), x - \bar{x}  \rangle  \ge \mu |x - \bar{x}|^2, \qquad \text{when} \quad |x - \bar{x}| \ge R,
		\end{equation}
        where $R$ is the same as in point (ii).
	\end{enumerate}
	
\end{assumption}

\begin{remark} \label{contraction_remark_contractivity_at_infinity}
	As a consequence of Proposition \ref{density_OU_smoothness_proposition}, due to \citet[Proposition 3.1]{conforti2023score}, and Assumption \ref{assumption_score_strong_monotonicity_new_semi_convexity}-(i),  we have that, for any $t \in (0,T)$, the map $x \mapsto \nabla p_t(x)$ is  continuously differentiable, and for any $x \in \mathbb{R}^d$, the map $t \mapsto p_t(x)$ is continuously differentiable on $(0,T]$.
	Moreover, Assumption \ref{assumption_score_strong_monotonicity_new_semi_convexity} implies that the processes in \ref{eq:backwardproc_real_initial_condition_introduction} and \ref{Y_hat_auxiliary} have a unique strong solution.
\end{remark}

\noindent Next, we consider the following assumption on the approximating function $s$, which is also adopted in \citet[Assumption 3.a]{bruno2025on}.

\stepcounter{assumption}
\begin{intassumption} \label{Assumption_2_without_derivative}
	The function $s : [0,T] \times \mathbb{R}^M \times \mathbb{R}^d \rightarrow \mathbb{R}^d $ is continuously differentiable in $x \in \mathbb{R}^d$. Let $D_1 : \mathbb{R}^M \times \mathbb{R}^M   \rightarrow \mathbb{R}_{+}$, $D_2 :  [0,T]  \times  [0,T]  \rightarrow \mathbb{R}_{+}$ and $D_3 : [0,T] \times [0,T] \rightarrow \mathbb{R}_{+}$ be such that $\int_{\epsilon}^T  \int_{\epsilon}^T D_2(t,\bar{t}) \ \text{d}t \ \text{d}\bar{t} < \infty $ and $\int_{\epsilon}^T  \int_{\epsilon}^T  D_3(t, \bar{t}) \ \text{d}t \ \text{d}\bar{t} < \infty $. For $\alpha \in \left[\frac{1}{2},1 \right]$ and for all $t, \bar{t} \in [0,T]$, $x, \bar{x} \in \mathbb{R}^{d}$, and $\theta, \bar{\theta} \in \mathbb{R}^M$, we have that
	\begin{equation*}
		\begin{split}
			| s(t, \theta,x) - s(\bar{t}, \bar{\theta},\bar{x}) | & \le  D_1(\theta, \bar{\theta})  |t - \bar{t}|^{\alpha} + D_2(t, \bar{t})  | \theta - \bar{\theta} |
			+ D_3(t,\bar{t})  |x -\bar{x} |,
		\end{split}		
	\end{equation*}
	where $D_1$, $D_2$ and $D_3$ have the following growth in each variable: i.e., there exist $\mathsf{K}_1$, $\mathsf{K}_2$, and $\mathsf{K}_3>0$ such that
	for each $t,\bar{t} \in [0,T]$ and $ \theta, \bar{\theta} \in \mathbb{R}^M$,
	\begin{equation*}
		\begin{split}
			|D_1(\theta, \bar{\theta}) |  & \le \mathsf{K}_1 (1+ | \theta | + | \bar{\theta} | ),
			\qquad 	|D_2(t, \bar{t}) | 	 \le \mathsf{K}_2 (1+ | t |^{\alpha} + | \bar{t}|^{\alpha}),
			\\ 	|D_3(t,\bar{t}) | & \le \mathsf{K}_3 (1+ | t |^{\alpha} + | \bar{t}|^{\alpha}).
		\end{split}
	\end{equation*}
\end{intassumption}
\begin{remark}
Assumption~\ref{Assumption_2_without_derivative} requires that the approximating function $s$ is Lipschitz continuous in both the input variable $x$ and the parameter $\theta$. In time $t$, it allows $s$ to be either H{\"o}lder continuous (for $\alpha \in [\frac{1}{2},1)$) or Lipschitz continuous (for $\alpha = 1$). This relaxed continuity in $t$ for the drift term of \ref{continuous_time_EM_version} is standard for the Euler--Maruyama schemes for SDEs. Crucially, we show in Theorems~\ref{main_theorem_general} and~\ref{main_theorem_general_better_rate_of_convergence} that this weaker condition in $t$ still guarantees convergence of the generative algorithm to $\pi_{\mathsf{D}}$.
	As noted by \citet[Remark 6]{bruno2025on} in the context of neural network-based approximations, Assumption~\ref{Assumption_2_without_derivative}, when $\alpha = 1$, covers the case where $s$ is implemented as a neural network with a hyperbolic tangent or sigmoid activation function at the final layer.
Moreover, Assumption \ref{Assumption_2_without_derivative} implies that the process in \ref{Y_auxiliary_theta_hat}, \ref{eq:backwardprocemdisc_old}, and \ref{continuous_time_EM_version} have a unique strong solution.
\end{remark}
\begin{remark} \label{remark_growth_estimate_neural_network}
	Let $ \mathsf{K}_{\text{Total}} := \mathsf{K}_1+\mathsf{K}_2+\mathsf{K}_3+  |  s(0, 0,0) |>0$. Using Assumption \ref{Assumption_2_without_derivative}, one obtains
	\begin{equation*}
		\begin{split}
			|s(t, \theta,x)| & \le   \mathsf{K}_{\text{Total}} (1+ | t |^{\alpha} ) (1+ | \theta | + |x| ).
		\end{split}
	\end{equation*}
\end{remark}
\noindent The proof of Remark \ref{remark_growth_estimate_neural_network} can be found, e.g., in \citet[Appendix D.3]{bruno2025on}.
By imposing an additional condition on the gradient of $s$ in Assumption \ref{Assumption_2_without_derivative}—as done in \citet[Assumption 3.b]{bruno2025on}—, we obtain the optimal convergence rate established in Theorem \ref{main_theorem_general_better_rate_of_convergence} below.	
\begin{intassumption} \label{Assumption_2}
	\renewcommand\theintassumption{} 
	Let $s$ be as in Assumption \ref{Assumption_2_without_derivative} and there exists $\mathsf{K}_4>0$ such that, for all $x, \bar{x} \in \mathbb{R}^d$ and  for any $k=1, \dots d$,
	\begin{equation*} 
		| \nabla_x s^{(k)}(t,\theta,x) - \nabla_{\bar{x}} s^{(k)} (t,\theta,\bar{x})  | \le  \mathsf{K}_4 (1+2 |t|^{\alpha}) | x -\bar{x}|.
	\end{equation*}
\end{intassumption}

\noindent For the following assumption on the score approximation, we let  $\hat{\theta}$ be as in Assumption \ref{general_assumption_algorithm} and we let $ (Y_t^{\text{aux}})_{t \in [0,T]}$ be the auxiliary process defined in \ref{Y_auxiliary_theta_hat}. 
\begin{assumption} \label{assumption_equivalence_global_minimiser_epsilon}
	There exists $\varepsilon_{\text{SN}} >0$ such that
	\begin{equation}\label{score_error_auxiliary}
		\mathbb{E}  \int_0^{T-\epsilon}   | \nabla \log p_{T - r}(Y_r^{\text{aux}}) - s(T-r, \hat{\theta} , Y_r^{\text{aux}})  |^2  \  \text{d} r <  \varepsilon_{\text{SN}} .
	\end{equation}
\end{assumption}
\begin{remark}
	Assumption \ref{assumption_equivalence_global_minimiser_epsilon} is now a standard assumption considered in the literature, see, e.g., \citet{ gao2023wasserstein,bruno2025on,strasman2025an, gentiloni2025beyond}, and its theoretical and practical soundedness is discussed, e.g., in \citet[Remark 7, 8, and 9]{bruno2025on}.
\end{remark}

	\subsection{Assumption \ref{assumption_score_strong_monotonicity_new_semi_convexity} and Weak Convexity of the Data Distribution} \label{assumptions_weak_convexity_section}

\noindent We start by introducing the definition of weak convexity, a concept that has been widely used in conjunction with coupling techniques to analyze the long-time behavior of gradient flows SDEs \citep{10.1214/22-AAP1927,conforti2024weak, conforti2023projected, conforti2023quantitative}\footnote{Recently, this notion has been used in the context of score-based generative models in \citet[Definition 3.1]{gentiloni2025beyond}.}, and we extend its application here to subgradients of $U$. 

\begin{defn} \label{weak_convexity_definition_gentiloni}
	The potential $U: \mathbb{R}^d \rightarrow \mathbb{R}$ is weakly convex if its weak convexity profile $\kappa_{U}: [0,\infty) \rightarrow \mathbb{R}$  defined as 
	\begin{equation} \label{definition_weak_convexity_profile_new}
		\kappa_{U} (r) = \inf_{x, \bar{x} \in \mathbb{R}^{d}:  \ |x-\bar{x} |=r} \left \{  \frac{\langle h(x) - h(\bar{x} ), x-\bar{x}   \rangle }{|x-\bar{x} |^2}    \right \},
	\end{equation} 
   where $h(x) \in \partial U(x)$ is the subgradient of $U$ at $x \in \mathbb{R}^d$,
	satisfies 
	\begin{equation} \label{weak_convexity_inequality_assumption_new}
		\kappa_{U} (r) \ge \beta - r^{-1} f_L(r), \quad \text{for all} \   \ r>0,
	\end{equation} 
	for some constants $\beta, L >0$,  
	where the function $ 	f_L: [0, \infty] \rightarrow [0, \infty]$ is defined as
	\begin{equation} \label{definition_f_L_assumption_new}
		f_L(r) = 2 L^{1/2} \tanh((r L^{1/2} )/2).
	\end{equation}
\end{defn}
\begin{remark} \label{Remark_notion_weak_convexity}
 The weak convexity profile $\kappa_U$ in \ref{definition_weak_convexity_profile_new} serves as an averaged/integrated convexity
lower bound for the potential $U$, evaluated over pairs of points separated by a distance $r > 0$.  Unlike the standard convexity condition $\kappa_U \ge 0$, which characterizes convex (or log-concave) potentials, Definition \ref{weak_convexity_definition_gentiloni} allows $\kappa_U$ to vary with $r$, thereby admitting non-uniform lower bounds. This generalization yields a broader and more flexible notion of convexity that extends beyond classical log-concavity, making Definition~\ref{weak_convexity_definition_gentiloni} substantially weaker than log-concavity \citep{10.1214/22-AAP1927,conforti2024weak, conforti2023projected, conforti2023quantitative, gentiloni2025beyond}.
\end{remark}

\noindent We modify \citet[Lemma 5.9]{conforti2023projected} to our setting, namely when  $\beta >0$\footnote{See \citet[Lemma B.4]{gentiloni2025beyond} for a similar statement.} to have an explicit expression of the weak convexity constant at each $t \in (0,T]$.
\begin{lem}\citep[Modification of Lemma 5.9]{conforti2023projected} \label{Lemma_contraction_Conforti}
	Assume that U is weakly convex as in Definition  \ref{weak_convexity_definition_gentiloni} and fix $t \in (0,T]$. Then, the function $x \mapsto - \log p_t(x)$ is weakly convex with weak convexity profile $\kappa_{- \log p_t(x)}$ satisfying
	\begin{equation*}
		\kappa_{- \log p_t}(r) \ge \frac{\beta }{\beta + (1-\beta)e^{-2t}} - \frac{e^{-t}}{\beta + (1-\beta)e^{-2t}} \frac{1}{r} f_{L}\left(\frac{e^{-t}}{\beta + (1-\beta)e^{-2 t}} r  \right).
	\end{equation*}
	In particular, the score function satisfies
	\begin{equation}  \label{contraction_score_assumption}
		\langle  \nabla  \log p_t(x)   - \nabla \log p_t(\bar{x}), x - \bar{x} \rangle  \le - \widehat{C}_t  | x - \bar{x}|^2,  \quad \text{for} \  x, \bar{x} \in \mathbb{R}^{\text{d}},
	\end{equation}
	with 
	\begin{equation} \label{weak_convexity_contraction_constant}
		\widehat{C}_t  =  \frac{\beta }{\beta + (1-\beta)e^{-2t}} - \frac{e^{-2t}}{(\beta + (1-\beta)e^{-2t})^2} L.
	\end{equation}
\end{lem}

\noindent We show that Assumption \ref{assumption_score_strong_monotonicity_new_semi_convexity}-(ii) and Assumption \ref{assumption_score_strong_monotonicity_new_semi_convexity}-(iii)  are related to the notion of weak convexity (Definition \ref{weak_convexity_definition_gentiloni}) in the sense made precise in Proposition \ref{Weak_convexity_implies_contractivity_outside_the_ball_new} below.  An overview of the proof of  Proposition \ref{Weak_convexity_implies_contractivity_outside_the_ball_new} below can be found in Appendix \ref{Weak_convexity_Assumption_2_Proof_Appendix}.
\begin{prop}\label{Weak_convexity_implies_contractivity_outside_the_ball_new}
	Let the data distribution $\pi_{\mathsf{D}}$ be in Assumption \ref{assumption_score_strong_monotonicity_new_semi_convexity}, and let  $f_L$ and $L>0$ be as in Definition \ref{weak_convexity_definition_gentiloni}. 
	Then the potential $U$ is weakly convex as in Definition \ref{weak_convexity_definition_gentiloni} with
	\begin{equation} \label{weak_convexity_profile_mu}
		\begin{split}
			\kappa_{U} (r)  \ge  \mu - r^{-1} f_L(r), \quad \text{for all}    \ r>0,
		\end{split}
	\end{equation} 
	where $\mu>0$ in \ref{weak_convexity_profile_mu} is the strong convexity at infinity constant from Assumption \ref{assumption_score_strong_monotonicity_new_semi_convexity}-(iii). Conversely, if $U$ is weakly convex as in Definition \ref{weak_convexity_definition_gentiloni} with lower bound  \ref{weak_convexity_profile_mu}
	for some known constants $\mu$ and $ L >0$,
	then 
	\begin{enumerate}
		\item The potential $U$ is $\widetilde{\mu} $-strongly convex at infinity  with $\widetilde{\mu}:= \mu - R^{-1}f_L(R)>0$\footnote{We refer to the proof of Proposition \ref{Weak_convexity_implies_contractivity_outside_the_ball_new} in Appendix \ref{Weak_convexity_Assumption_2_Proof_Appendix} below for the derivation of this constant.}, such that for all $ x, \bar{x} \in \mathbb{R}^d$, we have
		\begin{equation}  \label{equation_assumption_strong_convexity_outside_ball}
			\langle h(x) -  h(\bar{x}), x - \bar{x}  \rangle  \ge \widetilde{\mu}  |x - \bar{x}|^2, \qquad \text{when} \quad |x - \bar{x}| \ge R,
		\end{equation}
		which holds for all $R>0$  when $\mu > L$ and for $R \ge R_0 =  \frac{2 z_0}{L^{1/2}}$ with $z_0$ being the solution of \ref{equation_z_mu_weak_convexity} when $\mu \le L$.
		\item  The potential $U$ is $K$-semiconvex, such that  there exists $K \ge 0$  for all $ x, \bar{x} \in \mathbb{R}^d$, 
		\begin{equation}  \label{equation_assumption_semiconvexity_inside_ball}
			\langle  h(x) - h(\bar{x}), x - \bar{x}  \rangle  \ge -K  |x - \bar{x}|^2, \qquad \text{when} \quad |x - \bar{x}| < R,
		\end{equation}
        where $R$ is the same as in point (2).
	\end{enumerate}
	
	
\end{prop}

\noindent As a consequence of Proposition \ref{Weak_convexity_implies_contractivity_outside_the_ball_new} and Lemma \ref{Lemma_contraction_Conforti}, one obtains the explicit form of $\widehat{C}_t$ in \ref{contraction_score_assumption} in our setting, which is given in the following corollary.
\begin{cor} \label{corollary_constant_contractivity_at_infinity_semiconvexity}
	Let $U$ be $K$-semiconvex as in Assumption \ref{assumption_score_strong_monotonicity_new_semi_convexity}-(ii)  and be $\mu$-strongly convex at infinity  as in Assumption \ref{assumption_score_strong_monotonicity_new_semi_convexity}-(iii) and fix $t \in (0,T]$. Then  
	\begin{equation}  \label{contraction_score_assumption_constant_semiconvexity}
		\langle  \nabla  \log p_t(x)   - \nabla \log p_t(\bar{x}), x - \bar{x} \rangle  \le -  \beta_t^{\text{OS}}  | x - \bar{x}|^2, \qquad \text{for} \quad x,   \bar{x} \in \mathbb{R}^d, 
	\end{equation}
	where \begin{equation} \label{weak_convexity_contraction_constant_semiconvexity_alpha_bar}
		\beta^{\text{OS}}_t   =  \frac{\mu }{\mu+ (1-\mu  )e^{-2t}} - \frac{e^{-2t}}{(\mu + (1-\mu  )e^{-2t})^2} L,
	\end{equation} 
	for some $L>0$ satisfying \ref{weak_convexity_profile_mu}.
\end{cor}
\begin{remark} \label{remark_limit_behaviour_contraction_constant}
	By Corollary \ref{corollary_constant_contractivity_at_infinity_semiconvexity} and the proof of Proposition \ref{Weak_convexity_implies_contractivity_outside_the_ball_new}, we have 
	\begin{equation} \label{limit_zero_contractivity_constant}
		\lim_{t \rightarrow 0} 	\beta^{\text{OS}}_t   = \mu -L < -K,
	\end{equation}
    which shows that $ - \beta^{\text{OS}}_t$ is not the lowest bound for the left-hand side of \ref{contraction_score_assumption_constant_semiconvexity}.
	We emphasize that the gap between the limit on the left-hand side of \ref{limit_zero_contractivity_constant} and the semiconvexity constant $K$ is due to the particular choice of $f_L$ in \ref{definition_f_L_assumption_new} in Proposition  \ref{Weak_convexity_implies_contractivity_outside_the_ball_new}. This gap may vanish if we replace $f_L$ with an appropriate function $f \in \mathcal{\widetilde{F}}$\citep[Section 2.1.2]{conforti2023quantitative}, \citep[Section 5.3.1]{conforti2023projected}, where
	\begin{equation*}
		\begin{split}
			\mathcal{\widetilde{F}} := \Bigg \{ & f \in C^2((0,\infty), \mathbb{R}_+): \ r \mapsto r^{1/2} f(r^{1/2}), \ \text{non-decreasing, concave, bounded such that} \\ & \quad  \lim_{r \downarrow 0 } r f(r) = 0, \ f^{'} \ge 0, \quad  2 f^{''} + f f^{'} \le 0   \Bigg \}.
		\end{split}
	\end{equation*}
Note that $\mathcal{\widetilde{F}} $ is non-empty and contains $r \mapsto 2 \tanh(r/2)$.
	For this reason, we use the constant $K+\mu$ as a proxy of the constant $L$ and replace \ref{weak_convexity_contraction_constant_semiconvexity_alpha_bar}  with the following monotonicity bound
	\begin{equation} \label{weak_convexity_contraction_constant_K_mu}
		\beta^{\text{OS}, K, \mu}_t   =  \frac{\mu }{\mu+ (1-\mu  )e^{-2t}} - \frac{e^{-2t}}{(\mu + (1-\mu  )e^{-2t})^2} (K+\mu).
	\end{equation} 
	Moreover, it holds that
	\begin{equation*}
		\lim_{t \rightarrow 0} 	 \beta^{\text{OS}, K, \mu}_t  =-K\footnote{Indeed, this shows that  - $\beta^{\text{OS}, K, \mu}_0 < -\beta^{\text{OS}}_0$ for the right-hand side of \ref{contraction_score_assumption_constant_semiconvexity} in Corollary \ref{corollary_constant_contractivity_at_infinity_semiconvexity}.}, 
	\end{equation*}
	and
	\begin{equation*}
		\lim_{t \rightarrow \infty} 	\beta^{\text{OS}}_t   =  \lim_{t \rightarrow \infty} \beta^{\text{OS}, K, \mu}_t  =1, 
	\end{equation*}
	which is consistent with $\pi_{\infty} \sim \mathcal{N}(0, I_d)$,  the invariant distribution of the OU process.  
\end{remark}

\noindent Using the explicit expression of \ref{weak_convexity_contraction_constant_K_mu}, we are able to find a time for which the integral of the monotonicity bound\footnote{Note that $	\beta^{\text{OS}, K, \mu}_t  $ is a function of time.} $	\beta^{\text{OS}, K, \mu}_t  $  is positive. The proof of the following result is postponed to Appendix \ref{Weak_convexity_Assumption_2_Proof_Appendix}.
\begin{prop} \label{Time_integral_contraction_constant_proposition}
	Let $\mu>0$ and $K \ge 0$. The time integral of	$\beta^{\text{OS}, K, \mu}_t  $ from Remark \ref{remark_limit_behaviour_contraction_constant} is
	\begin{equation} \label{expression_constant_A_t}
		\begin{split} 
			B(t,0,  \mu, K)  & = \int_{0}^{t}  \left(\frac{\mu }{\mu+ (1-\mu  )e^{-2s}} - \frac{e^{-2s}}{(\mu + (1-\mu  )e^{-2s})^2} (K + \mu) \right) \ \rmd s 
			\\ & =  \frac{1}{2} \left[ \log \left(  \mu (e^{2t}-1) +1    \right)
			+   \left( \frac{K  }{\mu } + 1 \right) \left( \frac{1}{\mu (e^{2t} -1) +1 } - 1 \right) \right] >0,
		\end{split}
	\end{equation}
	when $t> t^{\star} > \bar{t}:=  \ln \left(\sqrt{ 1+ \frac{K}{\mu^2} } \right)$ with $t^{\star}:= \inf \left \{ t>0:   B(t,0,  \mu, K)  > 0  \right \}$.
\end{prop}
\begin{remark}
	If we consider the case when $K=0$ in Assumption \ref{assumption_score_strong_monotonicity_new_semi_convexity}-(ii), then \ref{expression_constant_A_t} is satisfied for all $t >0$.
\end{remark}
\begin{figure}
    \centering
\includegraphics[width=0.75\linewidth]{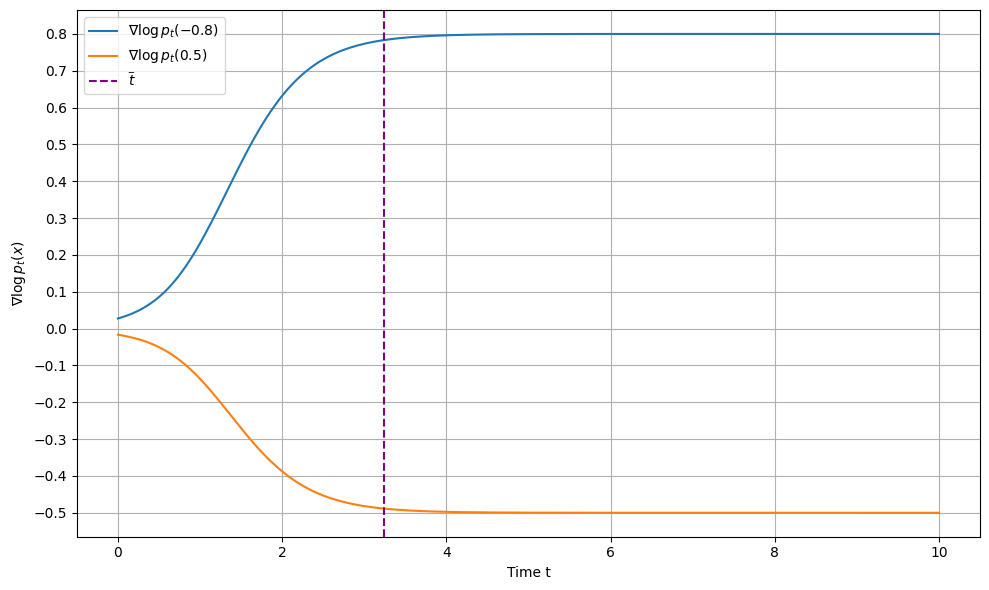}
    \caption{Score function \ref{score_function_numerical_example} for fixed $x$ and time $\bar{t}$.}
    \label{fig:Numerical_experiment_score_function}
\end{figure}
\begin{remark}
    We provide a numerical illustration of the critical time of Proposition \ref{expression_constant_A_t} when a non-log-concave distribution becomes log-concave in the case when $ \pi_{\mathsf{D}}$ is a one-dimensional Gaussian mixture with two equi-weighted modes $\eta = 2$, each mode having same variance $s^2=9 $, namely
    \begin{equation} \label{density_Gaussian_mixture_numerical_experiment_remark}
    \begin{split}
       \pi_{\mathsf{D}}(\rmd x) & = (2 (18 \pi )^{1/2})^{-1} \left(  \exp \left( - \frac{|x - 2|^2}{18}  \right) + \exp \left( - \frac{|x + 2 |^2}{18}  \right) \right) \rmd x.
    \end{split}
\end{equation}
For this choice of $\pi_{\mathsf{D}}$, the semiconvexity constant $K = \frac{2\eta^2}{(s^2)^2}=\frac{8}{81}$, and the strong convexity at infinity constant $\mu= \frac{s^2-2 \eta^2}{(s^2)^2} =\frac{1}{81} $. Here, $R \ge 0$. 
The score function of \ref{density_Gaussian_mixture_numerical_experiment_remark} is given by
\begin{equation} \label{score_function_numerical_example}
			\begin{split}
					\nabla 	\log p_t(x)
                    = -  \frac{x}{9 m_t^2 + \sigma_t^2} +  \frac{   2 m_t  }{9 m_t^2 + \sigma_t^2} \frac{   \exp \left\{  \frac{-\left( x   -  2 m_t      \right)^2}{2 (9 m_t^2 + \sigma_t^2)}    \right\}
					-    \exp \left\{  \frac{-\left( x   +  2 m_t      \right)^2}{2 (9 m_t^2 + \sigma_t^2)}    \right\} }{   \exp \left\{  \frac{-\left( x   -  2 m_t      \right)^2}{2 (9 m_t^2 + \sigma_t^2)}    \right\}
					+    \exp \left\{  \frac{-\left( x   +  2 m_t      \right)^2}{2 (9 m_t^2 + \sigma_t^2)}    \right\}}, \quad t \in [0,T], \ x \in \mathbb{R},
                    \end{split}
                    \end{equation}
\end{remark}
where $m_t=e^{-t}$ and $\sigma_t^2=1-e^{-2t}$ comes from the representation of the OU process in \ref{eq:OUdistribtuion}. Figure \ref{fig:Numerical_experiment_score_function} displays the time behaviour of score function \ref{score_function_numerical_example} over the time interval $[0,10]$   for fixed values $x = -0.8$ and $x = 0.5$. We also indicate the time $\bar{t}=  \ln \left(\sqrt{ 1+ \frac{K}{\mu^2} } \right) \approx  3.2377 < t^{\star}$ from Proposition \ref{score_function_numerical_example}. The score function in Figure \ref{fig:Numerical_experiment_score_function} converges to $\nabla \log( \pi_{\infty}(x)) =-x$, where $\pi_{\infty} = \mathcal{N}(0,1)$.

	\subsection{Main Results - Optimal Data Dimensional Dependence and Rate of Convergence} \label{main_result_section}
\noindent	The main results are stated as follows. An overview of their proofs can be found in Appendix \ref{proof_of_main_Theorem_appendix_section}. 
\begin{thm} \label{main_theorem_general}
	Let Assumptions \ref{general_assumption_algorithm}, \ref{assumption_score_strong_monotonicity_new_semi_convexity},  \ref{Assumption_2_without_derivative} and  \ref{assumption_equivalence_global_minimiser_epsilon} hold.
	Then, there exist constants $C_1$, $C_2$, $C_3$ and $C_4 >0$ such that	for any $T>0$ and $\gamma, \epsilon \in (0,1)$, 
	\begin{equation} \label{final_bound_first_inequality_statement_theorem}
		\begin{split}
			W_2(\mathcal{L}(Y_{J}^{\text{EM}}),\pi_{\mathsf{D}})  \le C_1 \sqrt{  \epsilon} + C_2 e^{   - 2  	\int_{\epsilon}^{T}  	\beta^{\text{OS}, K, \mu}_t   \ \rmd t   -\epsilon} +  C_3(T,\epsilon)  \sqrt{\varepsilon_{\text{SN}}} + C_4(T,\epsilon)  \gamma^{1/2},
		\end{split}
	\end{equation}
	where $C_1$, $C_2$, $C_3$ and $C_4$ are given explicitly in Table \ref{tab:convconst_general} (Appendix \ref{Table_of_constants_appendix}), $	\beta^{\text{OS}, K, \mu}_t $ is defined in \ref{weak_convexity_contraction_constant_K_mu}, and its integral is computed in Proposition \ref{Time_integral_contraction_constant_proposition}.
	In addition, the  result in \ref{final_bound_first_inequality_statement_theorem} implies that for any $\delta>0$, if we choose $0 < \epsilon<\epsilon_{\delta}$, $T>T_{\delta}$, $ 0 < \varepsilon_{\text{SN}} < \varepsilon_{\text{SN}, \delta}$ and $0< \gamma < \gamma_{\delta}$ with  $\epsilon_{\delta}$, $T_{\delta}$, $\varepsilon_{\text{SN}, \delta}$, and $\gamma_{\delta}$ given in Table \ref{tab:convconst_general}, then
	\begin{equation*}
		W_2(\mathcal{L}(Y_{J}^{\text{EM}}),\pi_{\mathsf{D}}) <~\delta.
	\end{equation*}		
\end{thm}
\begin{remark}
	The constant $C_4(T,\epsilon)$ in the error bound in \ref{final_bound_first_inequality_statement_theorem} contains the optimal dependence of the data dimension, i.e. $O(\sqrt{d})$, which has been found under the more strict assumption of strong-log concavity of $\pi_{\mathsf{D}}$ in \citet[Theorem 1 and Remark 12]{bruno2025on}.  However, the optimal dependence of the dimension is achieved at the expenses of a worst rate of covergence of order $1/2$. 
\end{remark}	

\noindent The optimal rate of convergence of order $\alpha \in [\frac{1}{2},1]$ for the Euler or Milstein scheme of SDEs with constant diffusion coefficients can be attained in Theorem \ref{main_theorem_general} provided that $	\mathbb{E} [ | \hat{\theta} |^4] < \infty$ and that  Assumption \ref{Assumption_2_without_derivative} is replaced by Assumption \ref{Assumption_2},  as stated in Theorem \ref{main_theorem_general_better_rate_of_convergence} below.

%

\begin{thm} \label{main_theorem_general_better_rate_of_convergence}
	Let Assumptions \ref{general_assumption_algorithm}, \ref{assumption_score_strong_monotonicity_new_semi_convexity},  \ref{Assumption_2} and  \ref{assumption_equivalence_global_minimiser_epsilon} hold, and assume that $	\mathbb{E} [ | \hat{\theta} |^4] < \infty$
	Then, there exist constants $C_1$, $C_2$, $C_3$ and $\widetilde{C}_4 >0$ such that	for any $T>0$ and $\gamma, \epsilon \in (0,1)$, 
	\begin{equation} \label{final_bound_first_inequality_statement_theorem_better_rate_of_convergence}
		\begin{split}
			W_2(\mathcal{L}(Y_{J}^{\text{EM}}),\pi_{\mathsf{D}})   \le C_1 \sqrt{  \epsilon} + C_2 e^{  - 2  		\int_{\epsilon}^{T}  	\beta^{\text{OS}, K, \mu}_t   \ \rmd t -\epsilon} +  C_3(T,\epsilon)  \sqrt{\varepsilon_{\text{SN}}} + \widetilde{C}_4(T,\epsilon)  \gamma^{\alpha},
		\end{split}
	\end{equation}
	where $C_1$, $C_2$, $C_3$ and $\widetilde{C}_4$ are given explicitly in Table \ref{tab:convconst_general} (Appendix \ref{Table_of_constants_appendix}), $	\beta^{\text{OS}, K, \mu}_t $ is defined in \ref{weak_convexity_contraction_constant_K_mu}, and its integral is computed in Proposition \ref{Time_integral_contraction_constant_proposition}.
	In addition, the  result in \ref{final_bound_first_inequality_statement_theorem_better_rate_of_convergence} implies that for any $\delta>0$, if we choose $0 < \epsilon<\epsilon_{\delta}$, $T>T_{\delta}$, $ 0 < \varepsilon_{\text{SN}} < \varepsilon_{\text{SN}, \delta}$ and $0< \gamma < \widetilde{\gamma}_{\delta}$ with  $\epsilon_{\delta}$, $T_{\delta}$, $\varepsilon_{\text{SN}, \delta}$, and $\widetilde{\gamma}_{\delta}$ given in Table \ref{tab:convconst_general}, then
	\begin{equation} \label{final_bound_first_inequality_statement_theorem_better_rate_of_convergence_small_error_constants}
		W_2(\mathcal{L}(Y_{J}^{\text{EM}}),\pi_{\mathsf{D}}) <~\delta.
	\end{equation}		
\end{thm}
\begin{remark}
The constant $\widetilde{C}_4$, explicitly given in Table \ref{tab:convconst_general} (Appendix \ref{Table_of_constants_appendix}) exhibits a linear dependence on the data dimension, i.e.,  $O(d)$. This scaling arises from the explicit Milstein scheme developed in \citet{kumar2019milstein}, which relies on  Assumption \ref{Assumption_2} and is leveraged in the proof of Theorem \ref{main_theorem_general_better_rate_of_convergence} to achieve the optimal convergence rate of order $\alpha \in [\frac{1}{2},1]$. This explicit Milstein scheme requires control on the fourth moment of the one-step discretization, see for instance, Lemma \ref{lem:distance_EM_scheme} (Appendix \ref{proof_of_main_Theorem_appendix_section}) which enables a convergence rate in $W_2$ consistent with the known optimal rate of convergence for the Euler or Milstein scheme of SDEs with constant diffusion coefficients. However, this comes at the cost of a worse dependence on the data dimension.
\end{remark}

\subsection{Examples of potentials satisfying by Assumption \ref{assumption_score_strong_monotonicity_new_semi_convexity}} \label{section_examples_data_distributions}
\noindent	We present several examples to demonstrate the wide applicability of our Assumption \ref{assumption_score_strong_monotonicity_new_semi_convexity} to a broad class of data distributions, some of which are not covered by previous results in Wasserstein distance of order two \citep{gentiloni2025beyond,strasman2025an, gao2023wasserstein,bruno2025on,tang2024contractive,yu2025advancing}.

\subsubsection{Symmetric modified half-normal distribution} \label{one_dimensional_case_Laplace_distribution_section}
\noindent	 We consider the case of a one-dimensional symmetric modified half-normal distribution
\begin{equation} \label{onedimensional_other_distribution_pdf_modified}
	\begin{split}
		\pi_{\mathsf{D}} (\rmd x) =  	 \frac{ \sqrt{\xi} \exp\left(- \xi x^2 - |x|  \right) }{\Psi \left( \frac{1}{2} , \frac{-1}{\sqrt{\xi}}  \right)} \  \rmd x, \quad x \in \mathbb{R},
	\end{split}
\end{equation}
for some unknown $\xi>0$ and normalizing constant  \begin{equation*}
	\begin{split}
		\Psi \left( \frac{1}{2} , \frac{-1}{\sqrt{\xi}}  \right) & := \sum_{n=0}^{\infty} \frac{\Gamma \left( \frac{1}{2} + \frac{n}{2}\right)}{\Gamma(n)} \frac{(-1)^n \xi^{-n/2}}{n! },
	\end{split}
\end{equation*}
where $\Gamma(n)$ is the Gamma function. We refer the reader to Appendix \ref{Modified_Half_Normal_Distribution_Appendix} for additional details about the derivation of \ref{onedimensional_other_distribution_pdf_modified}.  As highlighted in \citet[Section 2]{sun2023modified}, the modified half-normal distribution appears  in several Bayesian statistical methods as a posterior distribution to sample from in Bayesian Binary regression, analysis of directional data, and Bayesian graphical models. 

\noindent  Assumption \ref{assumption_score_strong_monotonicity_new_semi_convexity}-(i) is satisfied for $U(x)= \xi x^2 +  |x| $. In addition,  we have, for all $x, \bar{x} \in \mathbb{R}$
\begin{equation} \label{semiconvexity_Laplace_distribution}
	\begin{split}
		\langle  h(x) - h(\bar{x}), x - \bar{x} \rangle 
		& = 2 \left( \xi | x-\bar{x}|^2 +   (x-\bar{x}) \mathbbm{1}_{x>0,  \ \bar{x} <0} - (x-\bar{x}) \mathbbm{1}_{ x< 0,  \ \bar{x} >0  }   \right)
		\\ &  \ge 2 \xi | x-\bar{x}|^2,
	\end{split}
\end{equation}
which shows that Assumption \ref{assumption_score_strong_monotonicity_new_semi_convexity}-(ii) is verified for any $ K \ge 0$, and  Assumption \ref{assumption_score_strong_monotonicity_new_semi_convexity}-(iii) is verified for $\mu=2\xi $. Therefore, we can conclude that \ref{onedimensional_other_distribution_pdf_modified} satisfies Assumption \ref{assumption_score_strong_monotonicity_new_semi_convexity}.

\subsubsection{Multidimensional Gaussian mixture distribution}
\noindent We consider a multidimensional  Gaussian mixture data distribution with unknown mean and variance, i.e., 
\begin{equation} \label{multidimensional_Gaussian_mixture_pdf}
	\pi_{\mathsf{D}} (\rmd x) =  \sum_{i=1}^{I} \widetilde{\xi}_{i} \frac{1}{(2 \pi s_{i}^2)^{d/2}} \exp \left( - \frac{|x-\eta_{i}|^2}{2 s_{i}^2}  \right) \  \rmd x, \quad x \in \mathbb{R}^d,
\end{equation}
with $ s_{i}>0$, $\eta_{i} \in \mathbb{R}^d$, and $\widetilde{\xi}_{i} \in [0,1]$ for  $i \in \{1, \dots, I \}$ such that $\sum_{i=1}^J \widetilde{\xi}_{i} = 1$. The authors in \citet[Appendix A]{gentiloni2025beyond} show that the score function of \ref{multidimensional_Gaussian_mixture_pdf} is Lipschitz continuous and  $- \log 	\pi_{\mathsf{D}}$ is weakly convex. Therefore, Assumption \ref{assumption_score_strong_monotonicity_new_semi_convexity} is satisfied. In addition, the distribution \ref{multidimensional_Gaussian_mixture_pdf} covers also case of the double-well potential:
\begin{equation} \label{double_well_potential_multidimensional_case_section}
	U(x) =|x|^4 - |x|^2, \quad x \in \mathbb{R}^d,
\end{equation}
which is $2$-semiconvex and strongly convex at infinity.

\subsubsection{Multi-dimensional Potentials}

\noindent  Similarly as in Section \ref{one_dimensional_case_Laplace_distribution_section}, one can proves that the elastic net potential:
\begin{equation} \label{elastic_net_potential_example}
	U(x) = |x|^2 + \sum_{i=1}^d |x_i|, \quad x \in \mathbb{R}^d,
\end{equation}
satisfies Assumption \ref{assumption_score_strong_monotonicity_new_semi_convexity}.
Moreover, the following potential
\begin{equation} \label{U_multi_dimensional_example_K=0}
	U(x) = \max\left \{|x| , |x|^2 \right \} , \qquad x \in \mathbb{R}^d,
\end{equation}
verifies Assumption \ref{assumption_score_strong_monotonicity_new_semi_convexity} with $ K =0 $, $R=1$, and $\mu=2$ as well as the following non-convex potential presented in \citet[Example 4.2]{johnston2025performance}:
\begin{equation} \label{U_multi_dimensional_example}
	U(x) = \max\left \{|x| , |x|^2 \right \} - \frac{1}{2}|x|^2, \qquad x \in \mathbb{R}^d.
\end{equation}

\section{Related Work and Comparison} \label{sec:related_work}

\noindent In recent years, there has been a rapidly expanding body of research on the convergence theory of Score-based Generative Models. Existing works for convergence bounds can be divided into two main approaches, depending on the divergence or distance used.

\noindent The first approach focuses on $\alpha$-divergences, particularly the Kullback–Leibler (KL) divergence and Total Variation (TV) distance (e.g., \citet{benton2023linear,conforti2023score,yingxi2022convergence, li2024provable, block2020generative,de2021diffusion,lee2022convergence,li2023towards, lee2023convergence,chen2023improved,chen2023sampling,oko2023diffusion,liang2025low,yang2022convergence}), which  are the vast majority of the results available in the literature.  Crucially, bounds on KL divergence imply bounds on TV distance via Pinsker’s inequality, strengthening their wide applicability. We provide a brief and selective overview of some of the findings following this first approach. The results in TV distance in \citet{lee2022convergence} and in KL divergence  \citet{yingxi2022convergence} established convergence bounds characterized by polynomial complexity under the assumption that the data distribution satisfies a logarithmic Sobolev inequality and that the score function is Lipschitz continuous. By replacing the requirement that the data distribution satisfies a functional inequality with the assumption that $\pi_{\text{D}}$ has finite KL divergence with respect to the standard Gaussian and by assuming that the score function for the forward process is Lipschitz, the authors in \citet{chen2023sampling} managed to derive bounds in TV distance which scale polynomially in all the problem parameters. By requiring only the Lipschitzness of the score at the initial time rather than along the  full trajectory, the authors in \citet[Theorem 2.5]{chen2023improved} managed to establish, using an exponentially decreasing then linear step size, convergence bounds in KL divergence with quadratic dimensional dependence and logarithmic complexity in the Lipschitz constant. Later,  \citet{benton2023linear}  provided KL convergence bounds that are linear in the data dimension, up to logarithmic factors, by assuming finite second moments of the data distribution and employing early stopping. However, both the results of \citet[Theorem 2.5]{chen2023improved}  and  \citet[Theorem 1 and Corollary 1]{benton2023linear} still require the uniqueness of solutions for the backward SDE \ref{eq:backwardproc_real_initial_condition_introduction}, and therefore additional assumptions on the score function are needed. For further discussion on this point, we refer the reader to \citet[Section 4.2]{bruno2025on}. Assuming finite second moments and using an exponential integrator (EI) scheme with both constant and exponentially decaying step sizes, the authors in \citet[Corollary 2.4]{conforti2023score} derive a KL divergence bound with early stopping, which scales linearly in the data dimension up to logarithmic factors. Bounds in KL without early stopping have been derived in \citet{conforti2023score} for data distributions with finite Fisher information with respect to the standard Gaussian distribution. We note that this condition on $\pi_{\text{D}}$ stated in \citet[Assumption H2]{conforti2023score} still requires that the potential $U \in C^1(\mathbb{R}^d)$. The KL bounds provided in \citet[Theorem 2.1 and 2.2]{conforti2023score} scale linearly in the Fisher information when an EI discretization scheme with constant step size is used, and logarithmically in the Fisher information when an exponential-then-constant step size \citet[Theorem 2.3]{conforti2023score} is employed.

\noindent The second approach focuses on convergence bounds in Wasserstein distance, a metric which is often considered more practical and informative for estimation tasks. We can relate results following this approach with the results of the first approach only when 	$\pi_{\mathsf{D}}$ is a  strongly log-concave distribution. In this case, $W_2$-bounds in terms of KL divergence follow from an extension of Talagrand’s inequality \citep[Corollary 7.2]{gozlan2010transport}. However, for two general data distributions, there is no known relationship between their KL divergence and their $W_2$. Therefore, we cannot compare our findings in Theorem \ref{main_theorem_general} and Theorem \ref{main_theorem_general_better_rate_of_convergence} with the results derived following the first approach. One line of work within the second approach assumes (at least) strong log-concavity of the data distribution \citep{strasman2025an,gao2023wasserstein,bruno2025on,tang2024contractive,yu2025advancing}. Under this (strict) assumption, \citet[Remark 12]{bruno2025on}  achieved optimal data dimensional dependence, i.e., reaching $O(\sqrt{d})$. The recent bound in  \citet[Theorem D.1]{gentiloni2025beyond} exhibits similar scaling in $d$ while relaxing the strong log-concavity assumption on $\pi_{\mathsf{D}}$ to weakly log-concavity, but still requiring that the potential $\nabla^2 U$ exists (see, e.g., \citet[Proof of Proposition B.1 and B.2]{gentiloni2025beyond}). Our Assumption \ref{assumption_score_strong_monotonicity_new_semi_convexity} is much weaker than this requirement and it allows to consider the case of potentials with discontinuous gradients covering a wider range of distributions as outlined in Section \ref{section_examples_data_distributions}. Another line of work following this approach focuses on specific structural assumptions of the data distribution. For instance, convergence bounds in Wasserstein distance of order one with exponential dependence on the problem parameters have been obtained in  \citet{debortoli2022convergence} under the so-called manifold hypothesis, namely assuming that the target distribution is supported on a lower-dimensional manifold or is given by some empirical distribution. Under the same metric, the authors in \citet{mimikos-stamatopoulos2024scorebased} provide a convergence analysis when the data distribution is defined on a torus. Under the $W_2$ metric, \citet{wang2024wasserstein} derive convergence bounds assuming that the tail of $\pi_{\mathsf{D}}$ is Gaussian and that $U \in C^2$, which is a stronger condition than merely requiring $\nabla U$ to be Lipschitz. We summarize in Table \ref{Table_comparison_upper_bounds} and  the best results obtained in $W_2$, i.e., \citet{bruno2025on,gentiloni2025beyond} and compare with our best result, which scale polynomially in the data dimension, i.e. $O(\sqrt{d})$ in Theorem \ref{main_theorem_general}. As mentioned in the Introduction,
previous $W_2$ bounds \citep{gentiloni2025beyond, gao2023wasserstein, strasman2025an, tang2024contractive} require the stepsize $\gamma$ of the generative algorithm to be tuned based on quantities that are often difficult to compute in practice, such as the Lipschitz or strong convexity constant of the data distribution, which can lead to very small stepsizes. In contrast, Theorem~\ref{main_theorem_general} and Theorem~\ref{main_theorem_general_better_rate_of_convergence} impose no such restrictions, making them more suitable for practical use. Table~\ref{Table_comparison_stepsize} summarizes the assumptions on $\gamma$ in prior works and compares them with ours.

\noindent We close this section by briefly commenting on the choice of deriving our results in Wasserstein distance of order two. Beyond its theoretical relevance, this choice is motivated by practical considerations in generative modeling. First, the Wasserstein distance is often regarded as a more informative and robust metric for estimation tasks. Second, a widely used performance metric for evaluating the quality of images produced by generative models is the Fréchet Inception Distance (FID) \cite{heusel2017gans}, which measures the Fréchet distance between the distributions of generated and real samples, assuming Gaussian distributions. In particular, this Fréchet distance is equivalent to the Wasserstein-2 distance. Thus, providing convergence results under the Wasserstein-2 metric enhances the practical relevance of our theoretical findings.


\begin{table}[h!]
	\caption{Summary of previous bounds for $	W_2( \mathcal{L}(\widehat{Y}_{J}^{\text{EM}}), \pi_{\mathsf{D}}) $ and our result in Theorem \ref{main_theorem_general}. All the bounds assume that $\pi_{\text{D}}(\rmd x) \propto e^{- U(x)} \rmd x $  has finite second moments.} 
	\begin{tabular}{ |p{2.5cm} |p{11cm}| p{2cm}| p{2 cm} | } 
		\hline
		Assumption on $\pi_{\text{D}} $ 	  & Error bound & Reference  \\
		\hline
		$U$ strongly convex,  $ \nabla \log p_t(0) \in L^{2}([\epsilon,T])$, and Assumption \ref{assumption_equivalence_global_minimiser_epsilon}  & $  O(\sqrt{d}) \sqrt{  \epsilon} + O(\sqrt{d}) e^{      - 2 \widehat{	L}_{\text{MO}} (T-\epsilon) -\epsilon} +   O(e^{(1+  \zeta - 2 \widehat{ L}_{\text{MO}} )(T-\epsilon) } )  \sqrt{\varepsilon_{\text{SN}}} + O(\sqrt{d} e^{T^{2\alpha +1}} T^{2 \alpha +1} \widetilde{\varepsilon}^{1/2}_{\text{AL}})  \gamma^{1/2}, $
        
\qquad \qquad \qquad \qquad \qquad \qquad \qquad \qquad \qquad \qquad \qquad \qquad \qquad \qquad \qquad \qquad \qquad \qquad \qquad \qquad \qquad \qquad \qquad \qquad \qquad \qquad \qquad \qquad 

		with $\widehat{ L}_{\text{MO}}>0$ lower bound of the strongly convex constant of $U$, see e.g., \citet[Remark 4]{bruno2025on}.   & \citet[Remark 12]{bruno2025on}
		\\  \cline{1-3}  $U \in C^2(\mathbb{R}^d)$, weakly convex, and Assumption \ref{assumption_equivalence_global_minimiser_epsilon} &     $e^{(2L_U +5) \eta(\beta,L, (2L_U +5)^2 \gamma/2)} [e^{-T} W_2(  \pi_{\mathsf{D}}, \pi_{\infty}) + 4 \varepsilon_{\text{SN}} (T-\eta(\beta, L,0)) + \sqrt{2 \gamma}(  4L_U\sqrt{d} +6\sqrt{d} + \sqrt{d+\E[|X_0|^2]}  ) (T-\eta(\beta, L,0))  ]$, 

        \qquad \qquad \qquad \qquad \qquad \qquad \qquad \qquad \qquad \qquad \qquad \qquad \qquad \qquad \qquad \qquad \qquad \qquad \qquad \qquad \qquad \qquad \qquad \qquad \qquad \qquad \qquad \qquad 
        
	with $L_U \ge 0$ one-sided Lipschitz constant for $\nabla U$, see e.g., \citet[Assumption H1]{gentiloni2025beyond}, $\eta(\beta,L,\gamma)$ defined in  \cite[equation (29)]{gentiloni2025beyond}, and $\gamma < 2/(2L_U+5)^2$ .  &  \citet[Theorem D.1]{gentiloni2025beyond} 
		\\  \cline{1-3} 
		Assumption \ref{assumption_score_strong_monotonicity_new_semi_convexity} and Assumption \ref{assumption_equivalence_global_minimiser_epsilon}   &  $$  O(\sqrt{d}) \sqrt{  \epsilon} + O(\sqrt{d})  e^{   - 2  	\int_{\epsilon}^{T}  	\beta^{\text{OS}, K, \mu}_t   \ \rmd t   -\epsilon}$$ $$+   O(e^{(1+  \zeta)(T-\epsilon) - 2  	\int_{\epsilon}^{T}  	\beta^{\text{OS}, K, \mu}_t   \ \rmd t })   \sqrt{\varepsilon_{\text{SN}}} + O(\sqrt{d} e^{T^{2\alpha +1}} T^{3 \alpha +1} \widetilde{\varepsilon}^{1/2}_{\text{AL}})  \gamma^{1/2}.
		$$ &   Theorem \ref{main_theorem_general}
		\\ 
		\hline
	\end{tabular}
	\label{Table_comparison_upper_bounds}
\end{table}

\newpage

\begin{table}[h!]
	\caption{Summary of restrictions on the stepsizes of the generative algorithm $\widehat{Y}_{J}^{\text{EM}}$ used in the previous bounds for $	W_2( \mathcal{L}(\widehat{Y}_{J}^{\text{EM}}), \pi_{\mathsf{D}}) $ and our results in Theorem \ref{main_theorem_general} and Theorem \ref{main_theorem_general_better_rate_of_convergence}. All the bounds assume that $\pi_{\text{D}}(\rmd x) \propto e^{- U(x)} \rmd x $  has finite second moments.} 
	\begin{tabular}{ |p{2.5cm} |p{11cm}| p{2cm}| p{2 cm} | } 
		\hline
		Assumption on $\pi_{\text{D}} $ 	  & Restriction on the stepsize  & Reference  \\
		\hline
             $U$ strongly convex, $\nabla U$ Lipschitz continuous, and Assumption \ref{assumption_equivalence_global_minimiser_epsilon}  &  $$ 0< \gamma \le \min_{0 \le t \le T} \left( \frac{1-e^{-2t} (\frac{1}{m_0}-1)}{(1+e^{-2t} (\frac{1}{m_0}-1))(1+4 (\widetilde{L}(t))^2 + 2M_1 )} \right)  $$ and
        
          $$ 0< \gamma \le \min_{0 \le t \le T} \left( \frac{1+(\frac{1}{m_0}-1)e^{-2t}}{1 - (\frac{1}{m_0}-1)e^{-2t}}\right),$$ 
        
         with $m_0>0$ strong convexity constant for $U$, see e.g., \citet[Assumption 1]{gao2023wasserstein}, $M_1>0$ defined in \citet[Assumption 2]{gao2023wasserstein}, $\widetilde{L}(t) = \min_{0 \le t \le T} ((1-e^{-2t})^{-1}, e^{2t} L_0)$  Lipschitz
constant of $\nabla \log p_t(x)$, and $L_0>0$ Lipschitz constant for $\nabla U$, see e.g., \citet[Assumption 1]{gao2023wasserstein}.
        & \citet[Assumption 4]{gao2023wasserstein} used in \citet[Theorem 2]{gao2023wasserstein}  \\  \cline{1-3} 
		 $U$ strongly convex,  $ \nabla \log p_t(0) \in L^{2}([\epsilon,T])$, and Assumption \ref{assumption_equivalence_global_minimiser_epsilon}  &  $$ \gamma \in (0,1).$$  &  \citet[Remark 12]{bruno2025on}
        \\  \cline{1-3} 
		 $U$ strongly convex, and $\nabla U$ Lipschitz continuous  & 
        $$ 0< \gamma <  \frac{ C(T-t)}{ \left( \max_{t_j \le s \le t_{j+1} } L(T-s) \right) L(T-t)} e^{-(t_{j+1}-t_{j})}, $$ 
         
         with $\left \{t_j, \ 0 \le j \le J \right \}$ regular discretization  of $[0,T]$,  $ C_{T-t}$ and $ L_{T-t}$ are strong log-concavity and Lipschitz constant, respectively for $\nabla \log \frac{p_t}{\varphi_{\sigma^2}} $,  where $\varphi_{\sigma^2}$ is the density function of a mean zero Gaussian distribution with variance $\sigma^2 I_d$. &  \citet[Proposition C.3]{strasman2025an} used in \citet[Theorem 4.2]{strasman2025an}
         	\\  \cline{1-3}  $U$ strongly convex, $\nabla U$ Lipschitz continuous, and Assumption \ref{assumption_equivalence_global_minimiser_epsilon} &      $$0< \gamma < \min \left(  \frac{1}{2}, \ \frac{\kappa}{2 T(1+ \kappa)} \right ),$$ 
		
		with $\kappa>0$ strong convexity constant.  &  \citet[Theorem 3]{tang2024contractive} 
		\\  \cline{1-3}   $U \in C^2(\mathbb{R}^d)$, weakly convex, and Assumption \ref{assumption_equivalence_global_minimiser_epsilon}  &     $$0< \gamma < \frac{2}{(2L_U+5)^2},$$ 
		
	 	with $L_U \ge 0$ one-sided Lipschitz constant for $\nabla U$, see e.g., \citet[Assumption H1]{gentiloni2025beyond}.   &   \citet[Theorem D.1]{gentiloni2025beyond} 
		\\  \cline{1-3} 
	 	Assumption \ref{assumption_score_strong_monotonicity_new_semi_convexity}, and Assumption \ref{assumption_equivalence_global_minimiser_epsilon}   &   $$  \gamma \in (0,1).$$  &    Theorem \ref{main_theorem_general}, and Theorem \ref{main_theorem_general_better_rate_of_convergence}.
		\\ 
		\hline 
	\end{tabular}
	\label{Table_comparison_stepsize}
\end{table}

 \subsubsection*{Acknowledgments.} This work was supported by Innovate UK [grant number 10081810]. This work has received funding from the Ministry of Trade, Industry and Energy (MOTIE) and Korea Institute for Advancement of Technology (KIAT) through the International Cooperative R\&D program (No.P0025828). This work has been partially supported by project MIS 5154714 of the National Recovery and Resilience Plan Greece 2.0 funded by the European Union under the NextGenerationEU Program.

\bibliography{main}
\bibliographystyle{tmlr}

\appendix
\section*{Appendix}

	\section{Regularity of the Score Function} \label{regularity_appendix_section}

	\noindent We recall the following result to justify the smoothness of the map
	\begin{equation} \label{smoothness_map_OU}
		(0,T] \times \mathbb{R}^d \ni (t,x) \mapsto p_t(x) \in \mathbb{R}_{+},
	\end{equation}
	where $p_t$ density of the forward process defined in Section \ref{Technical_Background_section}.
	\begin{prop}\citep[Proposition 3.1]{conforti2023score} \label{density_OU_smoothness_proposition}
		Let $\pi_{\mathsf{D}}$ be absolutely continuous with respect to the Lebesgue measure, and denote its density by $p_0$. The map defined in \ref{smoothness_map_OU} is positive and solution of the following Fokker--Planck equation on $(0,T] \times \mathbb{R}^d$:
		\begin{equation*}
			\partial_t p_t(x) - \text{div} (x \ p_t)  - \Delta p_t(x) = 0, \quad \text{for} \ (t,x) \in  (0,T] \times \mathbb{R}^d.
		\end{equation*}
		Moreover, it belongs to $C^{1,2}((0,T] \times \mathbb{R}^d)$; i.e. for any $t \in (0,T)$, $x \mapsto p_t(x)$ is twice continuously differentiable, and for any $x \in \mathbb{R}^d$, $t \mapsto p_t(x)$ is continuously differentiable on $(0,T]$.
	\end{prop}

	\section{Further Details on Assumption \ref{assumption_score_strong_monotonicity_new_semi_convexity} and Weak Convexity of the Data Distribution} \label{Weak_convexity_Assumption_2_Proof_Appendix}
	
	\noindent We provide the proofs of Section \ref{assumptions_weak_convexity_section}.
	
	\begin{proof}[Proof of Proposition \ref{Weak_convexity_implies_contractivity_outside_the_ball_new}]
		We begin by considering that $\pi_{\mathsf{D}}$ satisfies Assumption  \ref{assumption_score_strong_monotonicity_new_semi_convexity}. Recall that $f_L$ is defined as in \ref{definition_f_L_assumption_new}.
		Note that $ r \mapsto r^{-1} f_L(r)$ is non-increasing on $(0, \infty)$ and $f_{L}^{'}(0)=L> r^{-1} f_L(r) $ for $r \in  (0, R]$. We look for $L>0$ satisfying
		\begin{equation} \label{small_ball_expression_to_verify}
			\inf_{r \in (0,R]} r^{-1} f_L(r) = 	R^{-1} f_L(R)  	= 2 R^{-1} L^{1/2} \tanh((R L^{1/2} )/2)= K + \mu.
		\end{equation}
		Equivalently, we look for $x = L^{1/2}R /2>0$ such that
		\begin{equation} \label{equation_x_hat_hyberbolic_tangent_new}
			x \tanh (x) =  \frac{K+ \mu}{4} R^2, \quad  
			\text{subject to} \quad x 	 > \frac{\sqrt{K + \mu}}{2} R,
		\end{equation}
		so as $L> K + \mu$.
		Note that $\tanh (x) \le x$ for all $x \ge 0$ . Therefore, if we choose $x=\frac{\sqrt{K + \mu}}{2} R$, then
		\begin{equation}
			\label{f_c_one_half_new_new_new}
			\frac{\sqrt{K + \mu}}{2} R  \tanh\left(  \frac{\sqrt{K + \mu}}{2} R \right) \le \frac{K + \mu}{4} R^2. 
		\end{equation}
		Using \ref{f_c_one_half_new_new_new} and  $\lim_{x \uparrow \infty} x \tanh (x) =\infty$, we deduce that there exists $x^{\star}>0$ such that 
		\begin{equation}
			x^{\star} \tanh (x^{\star}) = \frac{K+ \mu}{4} R^2,
		\end{equation}
		with $x^{\star} > \frac{\sqrt{K + \mu}}{2} R$, since  $ x \mapsto x \tanh (x) $ is non-decreasing on $(0, \infty)$. By Assumption \ref{assumption_score_strong_monotonicity_new_semi_convexity} and \ref{small_ball_expression_to_verify}, we have
		\begin{equation} \label{to_be_lower_bound_new}
			\begin{split}
				k_U(r)  & \ge   \mu- (K + \mu)
				\\ & \ge  \mu - r^{-1} f_L(r), \qquad \text{for} \ r \le R.
			\end{split}
		\end{equation}
		Moreover,
		\begin{equation*}
			\begin{split}
				k_U(r)  & \ge  \mu
				\\ & \ge   \mu- r^{-1} f_L(r),  \qquad \text{for} \ r > R,
			\end{split}
		\end{equation*}
		where it is used that $r^{-1} f_L(r) >0$ for all $r>0$. This proves the first part of the statement in Proposition \ref{Weak_convexity_implies_contractivity_outside_the_ball_new}, i.e. the lower bound \ref{weak_convexity_profile_mu}.
		
		\noindent 	Conversely, assume that  $U$ is weakly convex as in Definition \ref{weak_convexity_definition_gentiloni} with lower bound  \ref{weak_convexity_profile_mu} for some known constants $\mu$ and $ L >0$.  We look for $R$ such that 
		\begin{equation} \label{mu_tilde_kappa}
			\begin{split}
				\kappa_{U} (r) & \ge  \mu - r^{-1} f_L(r) 
				\\ & \ge  \mu - R^{-1} f_L(R)
				\\ & > 0,  \qquad  \qquad \qquad \qquad \forall \ r>R,
			\end{split}
		\end{equation} 
		where it is used that $r^{-1} f_L(r) $ is decreasing on $(0, \infty)$. Let $ \widetilde{\mu}:= \mu - R^{-1} f_L(R)$, so \ref{mu_tilde_kappa} becomes $	\kappa_{U} (r)  \ge \widetilde{\mu} >0$, for all $r>R$. One notes that
		\begin{equation} \label{inequality_proof_L_mu_Looking_R}
			\begin{split}
				\widetilde{\mu}  
				& =  \mu -  L \frac{ \tanh((R L^{1/2} )/2)}{(R L^{1/2} )/2} >0.
			\end{split}
		\end{equation}
		If $\mu> L$,  \ref{inequality_proof_L_mu_Looking_R}  is satisfied for all $R>0$.  
		If $\mu \le L$, \ref{inequality_proof_L_mu_Looking_R} holds for $R  \ge R_0$, where $R_0$ is the unique solution to 
		\begin{equation} \label{mu_R_second_statement_Proposition_3_8}
			\mu = \frac{2 L^{1/2}}{R} \tanh \left(\frac{R L^{1/2}}{2} \right).
		\end{equation}
		Let $z= \frac{R L^{1/2}}{2}$, then $R_0 = \frac{2 z_0}{L^{1/2}}$, where $z_0$ solves 
		\begin{equation} \label{equation_z_mu_weak_convexity}
			\frac{\tanh(z)}{z} = \frac{\mu}{L}.
		\end{equation}
		Since $\frac{\tanh(z)}{z} $ monotonically decreases from $1$ to $0$ as $z$ increases, a unique $z_0>0$ solving \ref{equation_z_mu_weak_convexity} exists for $\mu < L$. Therefore, \ref{mu_tilde_kappa} is satisfied for $R \ge R_0=\frac{2 z_0}{L^{1/2}}$. This proves that $U$ is $\widetilde{\mu} $-strongly convex at infinity, and therefore  \ref{equation_assumption_strong_convexity_outside_ball}.
		Using the assumption that $U$ is weakly convex as in Definition \ref{weak_convexity_definition_gentiloni}, one obtains that
		\begin{equation} \label{mu_tilde_kappa_L_compact_set}
			\begin{split}
				\kappa_{U} (r) & \ge  \mu - r^{-1} f_L(r) 
				\\ & \ge  \mu - L,  \qquad \qquad \text{for} \quad  r \le  R.
			\end{split}
		\end{equation}
		We distinguish two cases for the lower bound in \ref{mu_tilde_kappa_L_compact_set}. 
		If $\mu> L$, then $	\kappa_{U} (r)  \ge -K $ for $r \le R$ for all $R>0$ and $K \ge 0$. If $\mu \le L$, then, by setting $K=L-\mu$ in \ref{equation_z_mu_weak_convexity}, we have $	\kappa_{U} (r)  \ge -K $ for $r \le R$ for all $R >0$.  This proves that $U$ is $K$-semiconvex, and therefore  \ref{equation_assumption_semiconvexity_inside_ball}. This concludes the proof for the second part of the statement in Proposition \ref{Weak_convexity_implies_contractivity_outside_the_ball_new}.

	\end{proof}

	\begin{proof}[Proof of Proposition \ref{Time_integral_contraction_constant_proposition}]
		We look for $t^{\star}$ satisfying
		\begin{equation} \label{inequality_without_epsilon}
			\begin{split}
				B(t^{\star},0,  \mu, K)  =	\frac{1}{2} \left[ \log \left( \mu (e^{2t^{\star}}-1) +1    \right)
				+   \left( \frac{K  }{\mu } + 1 \right) \left( \frac{1}{\mu (e^{2t^{\star}} -1) +1 } - 1 \right) \right] >0.
			\end{split}
		\end{equation}
		Equivalently, we look for $x :=e^{2t^{\star}}-1$ such that
		\begin{equation} \label{definition_function_g_mu}
			\begin{split}
				g(\mu x  +1)  
				= \log (\mu x  +1) 
				-   \left( \frac{K  }{\mu } + 1 \right)   \frac{\mu x}{\mu x+1 }   >0.
			\end{split}
		\end{equation}
		Note that \ref{definition_function_g_mu} is satisfied for all $x>0$ when $K=0$. In addition, we have  
		\begin{equation}  \label{limit_behaviour_g}
			\begin{split}
				\lim_{x \rightarrow 0+}	g(\mu x  +1)  & =  0 , 
				\\   
				\lim_{x \rightarrow +\infty}	g(\mu x  +1)  & = \infty,
			\end{split}
		\end{equation}
		and
		\begin{equation} \label{non_negative_derivative}
			\begin{split}
				\frac{\rmd }{\rmd x} g(\mu x  +1)  & = \frac{\mu}{\mu x  +1} - \frac{K+ \mu}{(\mu x  +1)^2} \ge 0 \quad \text{when} \quad x \ge \frac{K}{\mu^2}.
				\\   \frac{\rmd^2 }{\rmd x^2} g(\mu x  +1)  & = - \frac{\mu^2}{(\mu x  +1)^2} + \frac{2(K+\mu) \mu}{(\mu x  +1)^3} \ge 0 \quad \text{when} \quad x \le \frac{2K}{\mu^2} + \frac{1}{\mu}.
			\end{split}
		\end{equation}
		By \ref{non_negative_derivative}, the function $g$ in \ref{definition_function_g_mu} has a minimum at $\frac{K}{\mu^2}$ and 
		\begin{equation*}
			g\left(\frac{K}{\mu}+1 \right)  =  \log \left(\frac{K}{\mu}  +1 \right) - \frac{K}{ \mu  } < 0,
		\end{equation*}
		for all $K,\mu>0$. By \ref{limit_behaviour_g} and \ref{non_negative_derivative}, there exists $x> \frac{K}{\mu^2}$ such that \ref{definition_function_g_mu} is strictly positive. Therefore, there exists $t^{\star} >  \ln \left(\sqrt{ 1+ \frac{K}{\mu^2} } \right)$ such that \ref{inequality_without_epsilon} holds.
	\end{proof}

	\section{Proof of the Main Results} \label{proof_of_main_Theorem_appendix_section}
	\noindent In this section, we present the proofs of Theorem \ref{main_theorem_general} and Theorem \ref{main_theorem_general_better_rate_of_convergence}. We begin by recalling an upper bound on the moments of the process $(\widehat{Y}_t^{\text{EM}})_{t \in [0,T-\epsilon]}$ defined in \ref{continuous_time_EM_version}, along with an estimate for its one-step discretization error. These results will be instrumental in the subsequent proofs.
	\begin{lem}\citep[Lemma 20]{bruno2025on}\label{lem:EMproc2ndbdgeneral}
		Let Assumptions \ref{general_assumption_algorithm} and \ref{Assumption_2_without_derivative} hold,
and suppose that  $	\mathbb{E} [ | \hat{\theta} |^p] < \infty$ for any $p \in [2, 4]$. Then, for any $t\in[0,T-\epsilon]$, 
		\begin{equation*}
			\sup_{0\leq s\leq t} \mathbb{E} \left[|\widehat{Y}_s^{\text{EM}} |^p \right]\leq C_{\mathsf{EM},p}(t),
		\end{equation*}
		where
		\begin{equation*}
			\begin{split}
				C_{\mathsf{EM},p}(t)
				&	: =   e^{t(3p-1 - \frac{2}{p} + 2^{2p-1} \mathsf{K}^p_{\text{Total}} (1+T^{\alpha p}))}
				\\ & \quad \times   \left( \mathbb{E} \left[ |\widehat{Y}_0^{\text{EM}}|^p \right] +   2^{3p-2}\mathsf{K}^p_{\text{Total}} t (1+ \mathbb{E} [  |  \hat{\theta} |^p  ]) (1+T^{\alpha p})  + \frac{2}{p} (pd+ p(p-2))^{\frac{p}{2}} t  \right) ,
			\end{split}
		\end{equation*}
	and $\mathsf{K}_{\text{Total}}$  is defined in Remark \ref{remark_growth_estimate_neural_network}.
	\end{lem}
	
	\begin{lem}\citep[Lemma 21]{bruno2025on}\label{lem:distance_EM_scheme}
		Let Assumptions \ref{general_assumption_algorithm} and \ref{Assumption_2_without_derivative} hold,
		and suppose that  $	\mathbb{E} [ | \hat{\theta} |^p] < \infty$ for any $p \in [2, 4]$. Then, for any $t\in[0,T-\epsilon]$, 
		\begin{equation*}
			\E\left[|\widehat{Y}_t^{\text{EM}}  - \widehat{Y}_{\lfloor t/\gamma \rfloor \gamma}^{\text{EM}}  |^p \right]\leq  \gamma^{\frac{p}{2}}  C_{\mathsf{EMose},p},
		\end{equation*}
		where
		\begin{equation*}
			\begin{split}
				C_{\mathsf{EMose},p} & :=   2^{p-1}  (C_{\mathsf{EM},p}(T) + \mathsf{K}^p_{\text{Total}} (1+  T^{\alpha p} )  (2^{3p-2}    C_{\mathsf{EM},p}(T) + 2^{4p-3} (1+ \E [|\hat{\theta}|^p])  ))
				\\ & \quad +   (d p(p-1))^{\frac{p}{2}},
			\end{split}
		\end{equation*}
		 $C_{\mathsf{EM},p}$ and $\mathsf{K}_{\text{Total}}$  are defined in Lemma \ref{lem:EMproc2ndbdgeneral} and in Remark \ref{remark_growth_estimate_neural_network}, respectively.
	\end{lem}
	
	\begin{proof}[Proof of Theorem \ref{main_theorem_general}]
		We derive the non-asymptotic estimate for $W_2(\mathcal{L}(Y_{J}^{\text{EM}}),\pi_{\mathsf{D}}) $ using the splitting 
		\begin{equation} \label{upper_bound_wasserstein_general}
			\begin{split}
				W_2(\mathcal{L}(Y_{J}^{\text{EM}}),\pi_{\mathsf{D}}) & \le
				W_2(\pi_{\mathsf{D}}, \mathcal{L}(Y_{t_{J}}))+W_2(\mathcal{L}(Y_{t_{J}}), \mathcal{L}(\widetilde{Y}_{t_{J}})) \\ & \quad + W_2( \mathcal{L}(\widetilde{Y}_{t_{J}}), \mathcal{L}(Y_{t_{J}}^{\text{aux}})) +W_2(\mathcal{L}(Y_{t_{J}}^{\text{aux}}), \mathcal{L}(Y_{J}^{\text{EM}})).
			\end{split}
		\end{equation}
		We provide upper bounds on the error made by the early stopping, i.e. $W_2(\pi_{\mathsf{D}}, \mathcal{L}(Y_{t_{J}}))$, the error made by approximating the initial condition of the backward process $Y_0 \sim \mathcal{L}(X_T)$ with $\widetilde{Y}_0 \sim \pi_{\infty} $, i.e.   $W_2(\mathcal{L}(Y_{t_{J}}), \mathcal{L}(\widetilde{Y}_{t_{J}}))$,  the error made by approximating the score function with $s$, i.e. $W_2( \mathcal{L}(\widetilde{Y}_{t_{J}}), \mathcal{L}(Y_{t_{J}}^{\text{aux}}))$,  and the discretisation error, i.e.  $W_2(\mathcal{L}(Y_{t_{J}}^{\text{aux}}), \mathcal{L}(Y_{J}^{\text{EM}}))$, separately.

		\paragraph{Upper bound on $W_2(\pi_{\mathsf{D}}, \mathcal{L}(Y_{t_{J}})).$} This bound can be established by following the same argument as in \cite[Proof of Theorem 10]{bruno2025on}, which relies on the representation of the OU process
		\begin{equation}\label{eq:OUdistribtuion}
			X_t\overset{\text{a.s.}}{=} m_t X_0+\sigma_t Z_t, \quad m_t = e^{-  t}, \quad \sigma_t^2 = 1-e^{-2  t}, \quad Z_t \sim \mathcal{N}(0,I_{d}),
		\end{equation}
		where $\overset{\text{a.s.}}{=}$ denotes almost sure equality. 
		Therefore, we have
		\begin{equation} \label{first_upper_bound_general}
			\begin{split}
				W_2(\pi_{\mathsf{D}}, \mathcal{L}(Y_{t_{J}}))
				& \leq 2 \sqrt{  \epsilon}  (\sqrt{  \mathbb{E}[|X_0  |^2]} + \sqrt{d} ),
			\end{split}
		\end{equation}
		where $t_{J} = T- \epsilon$.
		
		\paragraph{Upper bound on $W_2(\mathcal{L}(Y_{t_{J}}), \mathcal{L}(\widetilde{Y}_{t_{J}})).$}
		
		Using It\^o's formula,  we have, for any $t \in [0,T-\epsilon]$,
		\begin{equation} \label{Ito_formula_inequality_second_bound_new_second_new}
			\begin{split}
				\text{d} |Y_t -  \widetilde{Y}_t|^2
				&= 2\langle Y_t -  \widetilde{Y}_t,  Y_t +2  \nabla \log p_{T-t}(Y_t) -  \widetilde{Y}_t -2  \nabla \log p_{T-t}( \widetilde{Y}_t) \rangle \ \text{d}  t\\
				& = 2  | Y_t -  \widetilde{Y}_t|^2 \ \text{d}  t +4   \langle Y_t -  \widetilde{Y}_t,  \nabla \log p_{T-t}(Y_t)   -\nabla \log p_{T-t}( \widetilde{Y}_t)\rangle \ \text{d}  t.
			\end{split}
		\end{equation}
		By integrating and taking on both sides in  \ref{Ito_formula_inequality_second_bound_new_second_new}, we have
		\begin{equation} \label{Ito_formula_inequality_second_bound_new_second_new_integral_form_first}
			\begin{split}
            \E 	\left[ |Y_{t_{J}} -  \widetilde{Y}_{t_{J}} |^2 \right]
				& = \E 	\left[ |Y_{0} -  \widetilde{Y}_{0} |^2 \right]
				+ \int_0^{t_{J}}  2 	\E 	\left[ | Y_t -  \widetilde{Y}_t| ^2  \right]
				\ \rmd t
				\\ & \quad  +  \int_0^{t_{J}}   4  \E 	\left[ \langle Y_t -  \widetilde{Y}_t,  \nabla \log p_{T-t}(Y_t)   -\nabla \log p_{T-t}( \widetilde{Y}_t)\rangle   \right] \ \text{d}  t.
			\end{split}
		\end{equation}
		By integrating, taking expectations on both sides in \ref{Ito_formula_inequality_second_bound_new_second_new_integral_form_first}, using Corollary \ref{corollary_constant_contractivity_at_infinity_semiconvexity}, the representation \ref{eq:OUdistribtuion} with $Z_T \overset{\text{d}}{=} \widetilde{Y}_0$ (where $\overset{\text{d}}{=}$ denotes equality in distribution),  the inequality $1 - \sigma_t \le m_t$ for any $t \in [0,T]$, we have 
		\begin{equation}\label{eq:backwardsdecontr_Lipschitz_new_second_large}
			\begin{split}
				&	\mathbb{E}\left[|Y_{t_{J}} - \widetilde{Y}_{t_{J}} |^2  \right] \\ & \leq    \E 	\left[ |Y_{0} -  \widetilde{Y}_{0} |^2 \right]
				+ 2 \int_0^{t_{J}}  	\E 	\left[ | Y_t -  \widetilde{Y}_t| ^2  \right]
				\ \rmd t
				-4  \int_0^{t_{J}}  \beta^{\text{OS}}_{T-t}   \mathbb{E}\left[ | Y_t - \widetilde{Y}_t |^2   \right]\text{d}  t 
				\\ & \leq		\mathbb{E}[|Y_0- \widetilde{Y}_0 |^2]e^{    2 [t_{J} - 2   \int_0^{t_{J}}  \beta^{\text{OS}}_{T-t}   \ \rmd t ]  }
				\\ & =  	\mathbb{E}[|m_T X_0+ ( \sigma_T - 1) \widetilde{Y}_0 |^2]  e^{    2 [t_{J} - 2   \int_0^{t_{J}}  \beta^{\text{OS}}_{T-t}   \ \rmd t ]  }
				\\ & \leq 2  \left( \mathbb{E}[|X_0|^2] + d \right)  \ e^{    2 [t_{J} - 2   \int_0^{t_{J}}  \beta^{\text{OS}}_{T-t}   \ \rmd t ]  - 2 T  }.
			\end{split}
		\end{equation}
		Using \ref{eq:backwardsdecontr_Lipschitz_new_second_large},  Remark \ref{remark_limit_behaviour_contraction_constant}, and and $t_{J} = T- \epsilon$, we have
		\begin{equation} \label{final_second_bound_Wasserstein_general_new_second}
			\begin{split}
				W_2(\mathcal{L}(Y_{t_{J}}), \mathcal{L}(\widetilde{Y}_{t_{J}})) & \le \sqrt{	\mathbb{E} [ | Y_{t_{J}} - \widetilde{Y}_{t_{J}}  |^2] }
				\\ & \leq \sqrt{2}  ( 	\sqrt{\mathbb{E}[|X_0|^2]}+ \sqrt{d} ) e^{   - 2  	\int_{\epsilon}^{T}  	\beta^{\text{OS}, K, \mu}_t   \ \rmd t  -\epsilon}.
			\end{split}
		\end{equation}

		\paragraph{Upper bound on $W_2( \mathcal{L}(\widetilde{Y}_{t_{J}}), \mathcal{L}(Y_{t_{J}}^{\text{aux}})).$ }
		Using It{\^o}'s formula, we have, for $t \in [0,T-\epsilon]$,
		\begin{equation}\label{difference_aux_backward_process}
			\begin{split}
				\text{d} | \widetilde{Y}_t - Y_{t}^{\text{aux}} |^2 & = 2 \langle \widetilde{Y}_t - Y^{\text{aux}}_{t}, \widetilde{Y}_t + 2 \  \nabla \log p_{T - t}(\widetilde{Y}_t)  - Y^{\text{aux}}_{t}  -  2 \ s(T-t, \hat{\theta} , Y_t^{\text{aux}}) \rangle \  \text{d} t
				\\ & = 2 | \widetilde{Y}_t - Y^{\text{aux}}_{t}|^2 \ \text{d} t + 4 \ \langle \widetilde{Y}_t - Y^{\text{aux}}_{t},  \nabla \log p_{T - t}(\widetilde{Y}_t) -  \nabla \log p_{T - t}(Y_t^{\text{aux}})   \rangle \  \text{d} t
				\\ & \quad + 4 \ \langle \widetilde{Y}_t - Y^{\text{aux}}_{t},    \nabla \log p_{T - t}(Y_t^{\text{aux}}) -s(T-t, \hat{\theta} , Y_t^{\text{aux}})  \rangle \  \text{d} t.
			\end{split}
		\end{equation}
		By integrating and taking the expectation on both sides in \ref{difference_aux_backward_process},  using Corollary \ref{corollary_constant_contractivity_at_infinity_semiconvexity}, Young's inequality with $\zeta \in (0,1)$ and Assumption \ref{assumption_equivalence_global_minimiser_epsilon}, we have
		\begin{equation} \label{toward_third_error_bound_inequality}
			\begin{split}
				\mathbb{E} [  | \widetilde{Y}_{T-\epsilon} - Y_{T-\epsilon}^{\text{aux}} |^2 ] & = 2 \int_0^{T-\epsilon} 	\mathbb{E} [  | \widetilde{Y}_s - Y^{\text{aux}}_{s}|^2 ] \ \text{d} s
				\\ & \quad + 4 \int_0^{T-\epsilon} 	\mathbb{E} [   \langle \widetilde{Y}_s - Y^{\text{aux}}_{s},  \nabla \log p_{T - s}(\widetilde{Y}_s) -  \nabla \log p_{T - s}(Y_s^{\text{aux}})   \rangle  ]  \  \text{d} s
				\\ & \quad + 4 \int_0^{T-\epsilon} 	\mathbb{E} [   \langle \widetilde{Y}_s - Y^{\text{aux}}_{s},    \nabla \log p_{T - s}(Y_s^{\text{aux}}) - s(T-s, \hat{\theta} , Y_s^{\text{aux}})      \rangle ]  \  \text{d} s
				\\	 & \le  \int_0^{T-\epsilon}  2(1+  \zeta) \ 	\mathbb{E} [  | \widetilde{Y}_s - Y^{\text{aux}}_{s}|^2 ] \ \text{d} s
				\\ & \quad 	-4  \int_0^{t_{J}}  \beta^{\text{OS}}_{T-s}   \mathbb{E}\left[ | \widetilde{Y}_s - Y^{\text{aux}}_{s} |^2  \right]\text{d}  t  +  2 \zeta^{-1} \varepsilon_{\text{SN}}
				\\ & \le 2e^{2(1+  \zeta)(T-\epsilon) - 4 \int_0^{t_{J}}  \beta^{\text{OS}}_{T-t}	 \ \rmd t  } \    \zeta^{-1} \varepsilon_{\text{SN}}.
			\end{split}
		\end{equation}
		Using \ref{toward_third_error_bound_inequality}, Remark \ref{remark_limit_behaviour_contraction_constant}, and $t_{J} = T- \epsilon$, we have
		\begin{equation} \label{final_third_upper_bound}
			\begin{split}
				W_2(\mathcal{L}(\widetilde{Y}_{t_{J}}),\mathcal{L}(Y^{\text{aux}}_{t_{J}}))  & \le \sqrt{	\mathbb{E} [ |\widetilde{Y}_{t_{J}} - Y_{t_{J}}^{\text{aux}} |^2] }
				\\ &  \le   \sqrt{2  \zeta^{-1}} e^{(1+  \zeta)(T-\epsilon) - 2  	\int_{\epsilon}^{T}  	\beta^{\text{OS}, K, \mu}_t   \ \rmd t  }  \sqrt{   \varepsilon_{\text{SN}} }.
			\end{split}
		\end{equation}
		
		\paragraph{Upper bound on $W_2(  \mathcal{L}(Y_{t_{J}}^{\text{aux}}),\mathcal{L}(\widehat{Y}_t^{\text{EM}})).$ }
		Using It{\^o}'s formula, we have, for $t \in [0,T-\epsilon]$, 
		\begin{equation} \label{Fourth_term_upper_bound_Ito_formula}
			\begin{split}
				&	\text{d}	|Y_t^{\text{aux}} - \widehat{Y}_t^{\text{EM}} |^2
				\\ &= 2 \langle Y_t^{\text{aux}} - \widehat{Y}_t^{\text{EM}} , Y_t^{\text{aux}} + 2 \ s(T-t, \hat{\theta},Y_t^{\text{aux}}  ) - \widehat{Y}^{\text{EM}}_{\lfloor t/ \gamma\rfloor \gamma } - 2 \ s(T- \lfloor t/ \gamma\rfloor \gamma , \hat{\theta}, \widehat{Y}^{\text{EM}}_{\lfloor t/ \gamma\rfloor \gamma }  ) \rangle \ 	\text{d}t
				\\ &= 2 | Y_t^{\text{aux}} - \widehat{Y}_t^{\text{EM}} |^2  \ 	\text{d}t  + 2 \langle Y_t^{\text{aux}} - \widehat{Y}_t^{\text{EM}} , \widehat{Y}_t^{\text{EM}}  - \widehat{Y}^{\text{EM}}_{\lfloor t/ \gamma\rfloor \gamma }  \rangle \ 	\text{d}t
				\\ & \quad  + 4 \langle Y_t^{\text{aux}} - \widehat{Y}_t^{\text{EM}} , s(T-t, \hat{\theta},Y_t^{\text{aux}}  ) - s(T- t , \hat{\theta}, \widehat{Y}^{\text{EM}}_{t }  )  \rangle \ 	\text{d}t
				\\ & \quad + 4 \langle Y_t^{\text{aux}} - \widehat{Y}_t^{\text{EM}} ,s(T- t , \hat{\theta}, \widehat{Y}^{\text{EM}}_{t }  ) - s(T- \lfloor t/ \gamma\rfloor \gamma , \hat{\theta}, \widehat{Y}^{\text{EM}}_{\lfloor t/ \gamma\rfloor \gamma }  )  \rangle \ 	\text{d}t.
			\end{split}
		\end{equation}
		Integrating and taking the expectation on both sides in \ref{Fourth_term_upper_bound_Ito_formula}, using Young's inequality for $\zeta \in (0,1)$, Cauchy Schwarz inequality, Assumption \ref{Assumption_2_without_derivative}, Lemma \ref{lem:distance_EM_scheme}, and Remark \ref{Control_algorithm}, we have 
		\begin{equation} \label{Fourth_term_upper_bound_Ito_formula_continuation} 
			\begin{split}
				\E \left [ 	|Y_{T- \epsilon}^{\text{aux}} - \widehat{Y}_{T- \epsilon}^{\text{EM}} |^2  \right] & \le  (2+ 3\zeta) \int_0^{T-\epsilon} 	\E  [| Y_t^{\text{aux}} - \widehat{Y}_t^{\text{EM}} |^2 ]  \ 	\text{d}t + \zeta^{-1} \int_0^{T-\epsilon} \E [| \widehat{Y}_t^{\text{EM}}  - \widehat{Y}^{\text{EM}}_{\lfloor t/ \gamma\rfloor \gamma } |^2]  \ 	\text{d}t
				\\ & \quad  + 4  \mathsf{K}_3 (1+2T^{\alpha})   \int_0^{T-\epsilon} \E [ | Y_t^{\text{aux}} - \widehat{Y}_t^{\text{EM}} |^2 ] \rmd t 
				\\ & \quad + 2 \zeta^{-1} \int_0^{T-\epsilon} \E [| s(T- t , \hat{\theta}, \widehat{Y}^{\text{EM}}_{t }  ) - s(T- \lfloor t/ \gamma\rfloor \gamma , \hat{\theta}, \widehat{Y}^{\text{EM}}_{\lfloor t/ \gamma\rfloor \gamma }  |^2 ]     \rmd t 
			\\	& \le  (2+ 3\zeta + 4  \mathsf{K}_3 (1+2T^{\alpha})  ) \int_0^{T-\epsilon} 	\E  [| Y_t^{\text{aux}} - \widehat{Y}_t^{\text{EM}} |^2 ]  \ 	\text{d}t 
				\\ & \quad + \zeta^{-1}    \gamma (T-\epsilon)  C_{\mathsf{EMose},2}
				+ 8 \zeta^{-1} \gamma^{2 \alpha} (T-\epsilon)   \mathsf{K}_1^2  (1+4  \E[|\hat{\theta}|^2]  )
				\\ & \quad + 4 \zeta^{-1}  \mathsf{K}^2_3 (1+2T^{\alpha})^2 \int_0^{T-\epsilon} \E [| \widehat{Y}^{\text{EM}}_{t }  -\widehat{Y}^{\text{EM}}_{\lfloor t/ \gamma\rfloor \gamma }  |^2 ]     \rmd t 
				\\ & \le  (2+ 3\zeta + 4  \mathsf{K}_3 (1+2T^{\alpha})  ) \int_0^{T-\epsilon} 	\E  [| Y_t^{\text{aux}} - \widehat{Y}_t^{\text{EM}} |^2 ]  \ 	\text{d}t 
				\\ & \quad + \zeta^{-1}    \gamma (T-\epsilon)  C_{\mathsf{EMose},2} (1+4 \mathsf{K}^2_3 (1+2T^{\alpha})^2 )
				\\ & \quad + 8 \zeta^{-1} \gamma^{2 \alpha} (T-\epsilon)   \mathsf{K}_1^2  (1+8  \widetilde{\varepsilon}_{\text{AL}} + 8 | \theta^{*}|^2)
				\\ & \le e^{(2+ 3\zeta + 4  \mathsf{K}_3 (1+2T^{\alpha})  )(T-\epsilon)} 
				\\& \quad \times \Bigg( \zeta^{-1}    \gamma (T-\epsilon)  C_{\mathsf{EMose},2} (1+4 \mathsf{K}^2_3 (1+2T^{\alpha})^2 )
				\\ & \qquad \quad 
				+ 8 \zeta^{-1} \gamma^{2 \alpha} (T-\epsilon)   \mathsf{K}_1^2  (1+8  \widetilde{\varepsilon}_{\text{AL}} + 8 | \theta^{*}|^2) \Bigg).
			\end{split}
		\end{equation}
		Using \ref{Fourth_term_upper_bound_Ito_formula_continuation} and $t_{J}=T-\epsilon$, we have 
		\begin{equation} \label{fourth_upper_bound_general_improved_dimension_dependence}
			\begin{split}
				W_2(\mathcal{L}(Y_{T- \epsilon}^{\text{aux}}), \mathcal{L}(\widehat{Y}_{T- \epsilon}^{\text{EM}} ))
				& \le \gamma^{1/2} \zeta^{-1/2} (T-\epsilon)^{1/2}  e^{(1+ (3/2)\zeta + 2  \mathsf{K}_3 (1+2T^{\alpha})  )(T-\epsilon)} 
				\\& \quad \times     (   C^{1/2}_{\mathsf{EMose},2}(1+ 2  \mathsf{K}_3 (1+2T^{\alpha})    ) + 2 \sqrt{2}     \mathsf{K}_1  (1+8  \widetilde{\varepsilon}_{\text{AL}} + 8 | \theta^{*}|^2)^{1/2}). 
			\end{split}
		\end{equation}
		
		\paragraph{Final upper bound on $W_2(\mathcal{L}(Y_{J}^{\text{EM}}),\pi_{\mathsf{D}}).$} Substituting \ref{first_upper_bound_general}, \ref{final_second_bound_Wasserstein_general_new_second}, \ref{final_third_upper_bound}, and \ref{fourth_upper_bound_general_improved_dimension_dependence}  into \ref{upper_bound_wasserstein_general}, we have
		\begin{equation} \label{final_bound_first_inequality}
			\begin{split}
				W_2(\mathcal{L}(Y_{J}^{\text{EM}}),\pi_{\mathsf{D}}) & \le  (\sqrt{  \mathbb{E}[|X_0  |^2]} + \sqrt{d} ) 2 \sqrt{  \epsilon}
				\\ & \quad + \sqrt{2}  ( 	\sqrt{\mathbb{E}[|X_0|^2]}+ \sqrt{d} ) e^{   - 2  	\int_{\epsilon}^{T}  	\beta^{\text{OS}, K, \mu}_t   \ \rmd t  -\epsilon}
				\\ & \quad + \sqrt{2  \zeta^{-1}} e^{(1+  \zeta)(T-\epsilon) - 2 \int_{\epsilon}^{T}  	\beta^{\text{OS}, K, \mu}_t   \ \rmd t  }  \sqrt{   \varepsilon_{\text{SN}} }
				\\ & \quad + \gamma^{1/2} \zeta^{-1/2} (T-\epsilon)^{1/2}  e^{(1+ (3/2)\zeta + 2  \mathsf{K}_3 (1+2T^{\alpha})  )(T-\epsilon)} 
				\\& \qquad \times     (   C^{1/2}_{\mathsf{EMose},2}(1+ 2  \mathsf{K}_3 (1+2T^{\alpha})    ) + 2 \sqrt{2}     \mathsf{K}_1  (1+8  \widetilde{\varepsilon}_{\text{AL}} + 8 | \theta^{*}|^2)^{1/2}).
			\end{split}
		\end{equation}
		The bound for $W_2(\mathcal{L}(\widehat{Y}_{J}^{\text{EM}}),\pi_{\mathsf{D}}) $ in \ref{final_bound_first_inequality} can be made arbitrarily small by appropriately choosing parameters including $\epsilon,T, \varepsilon_{\text{SN}}$ and $\gamma$. More precisely, for any $\delta>0$, we first choose $ 0 <   \epsilon < \epsilon_{\delta}$ with $\epsilon_{\delta}$ given in Table \ref{tab:convconst_general} such that the first term on the right-hand side of \ref{final_bound_first_inequality} is
		\begin{equation} \label{first_delta_4_general}
			(\sqrt{  \mathbb{E}[|X_0  |^2]} + \sqrt{d} ) 2 \sqrt{  \epsilon} <\delta/4.
		\end{equation}
		Next, we choose $T > T_{\delta}$ with $T_{\delta}$ given in Table \ref{tab:convconst_general} such that the second term on the right-hand side of \ref{final_bound_first_inequality} is
		\begin{equation} \label{second_delta_4_general}
			\sqrt{2}  ( 	\sqrt{\mathbb{E}[|X_0|^2]}+ \sqrt{d} ) e^{   - 2  	\int_{\epsilon}^{T}  	\beta^{\text{OS}, K, \mu}_t   \ \rmd t   -\epsilon} <\delta/4.
		\end{equation}
		Next, we turn to the third term on the right-hand side of \ref{final_bound_first_inequality}. We choose $ 0 < \varepsilon_{\text{SN}} < \varepsilon_{\text{SN}, \delta}$ with $\varepsilon_{\text{SN}, \delta}$ given in Table \ref{tab:convconst_general} such that
		\begin{equation} \label{third_delta_4_general_first}
			\begin{split}
				\sqrt{2  \zeta^{-1}} e^{(1+  \zeta)(T-\epsilon) - 2 \int_{\epsilon}^{T}  	\beta^{\text{OS}, K, \mu}_t   \ \rmd t } \sqrt{   \varepsilon_{\text{SN}} } < \delta / 4.
			\end{split}
		\end{equation}
		Finally, we choose $ 0 < \gamma < \gamma_{\delta}$ with $\gamma_{\delta}$ given in Table \ref{tab:convconst_general} such that  the fourth term on the right-hand side of \ref{final_bound_first_inequality} is
		\begin{equation}  \label{fourth_delta_4_general}
			\begin{split}
				&  \gamma^{1/2} \zeta^{-1/2} (T-\epsilon)^{1/2}  e^{(1+ (3/2)\zeta + 2  \mathsf{K}_3 (1+2T^{\alpha})  )(T-\epsilon)} 
				\\& \qquad \times     (   C^{1/2}_{\mathsf{EMose},2}(1+ 2  \mathsf{K}_3 (1+2T^{\alpha})    ) + 2 \sqrt{2}     \mathsf{K}_1  (1+8  \widetilde{\varepsilon}_{\text{AL}} + 8 | \theta^{*}|^2)^{1/2}) <\delta/4.
			\end{split}
		\end{equation}
		Using \ref{first_delta_4_general}, \ref{second_delta_4_general}, \ref{third_delta_4_general_first}, and \ref{fourth_delta_4_general},  we obtain $ W_2(\mathcal{L}(\widehat{Y}_{J}^{\text{EM}}),\pi_{\mathsf{D}}) <\delta$.	
	\end{proof}
	
	\begin{proof}[Proof of Theorem \ref{main_theorem_general_better_rate_of_convergence}]
		Using the splitting \ref{upper_bound_wasserstein_general}, the proof follows along the same lines of the Proof of Theorem \ref{main_theorem_general} for the estimation of the error bounds of the terms $W_2(\pi_{\mathsf{D}}, \mathcal{L}(Y_{t_{J}}))$, $W_2(\mathcal{L}(Y_{t_{J}}), \mathcal{L}(\widetilde{Y}_{t_{J}}))$, and $ W_2( \mathcal{L}(\widetilde{Y}_{t_{J}}), \mathcal{L}(Y_{t_{J}}^{\text{aux}}))$. The error bound for $W_2(\mathcal{L}(Y_{t_{J}}^{\text{aux}}), \mathcal{L}(Y_{J}^{\text{EM}}))$ is derived along the same lines of \citet[Proof of Theorem 10]{bruno2025on}. Putting these four estimates together leads to \ref{final_bound_first_inequality_statement_theorem_better_rate_of_convergence} and \ref{final_bound_first_inequality_statement_theorem_better_rate_of_convergence_small_error_constants}.
	\end{proof}
	
	\section{Modified Half-Normal Distribution} \label{Modified_Half_Normal_Distribution_Appendix}
	\noindent In this section, we recall the probability density function of the modified half-normal distribution, see e.g., \citet{sun2023modified}, used in Section \ref{one_dimensional_case_Laplace_distribution_section} and defined as 
	\begin{equation} \label{Modified_Half_Normal_Distribution_Appendix_general_equation}
		\begin{split}
			g(x) = \frac{2 \xi^{\frac{\upsilon}{2}} x^{\upsilon -1} \exp\left(- \xi x^2 + \psi x  \right) }{\Psi \left( \frac{\upsilon}{2} , \frac{\psi}{\sqrt{\xi}}  \right)}, \qquad x \ge 0,
		\end{split}
	\end{equation}
	where $\upsilon , \xi >0$, $\psi \in \mathbb{R}$, and the normalizing constant
	\begin{equation*}
		\begin{split}
			\Psi \left( \frac{\upsilon}{2} , \frac{\psi}{\sqrt{\beta}}  \right) & := \sum_{n=0}^{\infty} \frac{\Gamma \left( \frac{\upsilon}{2} + \frac{n}{2}\right)}{\Gamma(n)} \frac{\psi^n \xi^{-n/2}}{n! },
		\end{split}
	\end{equation*}
	is the Fox--Wright function \citep{fox1928asymptotic, wright1935asymptotic}. We point out that the half-normal distribution, truncated normal distribution, gamma distribution, and square root of the gamma distribution are all special cases of the modified Half-Normal distribution \ref{Modified_Half_Normal_Distribution_Appendix_general_equation}. The distribution \ref{onedimensional_other_distribution_pdf_modified} follows by taking the symmetric extension of \ref{Modified_Half_Normal_Distribution_Appendix_general_equation}, i.e. $g(|x|)/2$, and choosing $\upsilon=1$ and $\psi=-1$.

	\section{Table of Constants} \label{Table_of_constants_appendix}
	\noindent		Table~\ref{tab:convconst_general} displays full expressions for constants which appear in Theorem \ref{main_theorem_general} and Theorem \ref{main_theorem_general_better_rate_of_convergence}.
	\newpage
	\begin{table}[h] 
		\caption{Explicit expressions for the constants in  Theorem \ref{main_theorem_general} and Theorem \ref{main_theorem_general_better_rate_of_convergence}.}
		\renewcommand{\arraystretch}{2}
		\centering

		
		\scriptsize
		\begin{tabular}{@{}llllllllllllll@{}}
			\toprule
			
			\multicolumn{1}{c}{\begin{sc} Constant \end{sc}} &
			\multicolumn{2}{c}{\begin{sc} Dependency \end{sc}} &
			\multicolumn{3}{c}{\begin{sc}  Full Expression \qquad  \qquad\qquad\qquad\qquad\qquad\qquad \qquad\qquad\qquad\qquad\qquad \  \end{sc}}   \\
			
			
			
			\toprule
			$C_1$&  $O(\sqrt{d})$ &  & $ 2   ( \sqrt{\mathbb{E}[|X_0|^2]}+\sqrt{d}) $ \\\hline
			$C_2$ &  $O(\sqrt{d})$ & &  $\sqrt{2} \left(	\sqrt{\mathbb{E}[|X_0|^2]} + \sqrt{d} \right)  $  \\\hline
			$C_3(T,\epsilon)$ & $ O(e^{(1+  \zeta)(T-\epsilon) - 2 \int_{\epsilon}^{T}  	\beta^{\text{OS}, K, \mu}_t   \ \rmd t  }) $ & &  $  \sqrt{2  \zeta^{-1}} e^{(1+  \zeta)(T-\epsilon) - 2 	\int_{\epsilon}^{T}  	\beta^{\text{OS}, K, \mu}_t   \ \rmd t  } $
			\\\hline $C_{\mathsf{EM},2}(T)$ & $O (M e^{T^{2\alpha+1}}  T^{2 \alpha +1} \widetilde{\varepsilon}_{\text{AL}}) $ & &
			$
			\begin{aligned}
				&	e^{T(4 + 8 \mathsf{K}^2_{\text{Total}} (1+T^{2 \alpha }))}
				\\ & \quad \times
				( \mathbb{E} [ |\widehat{Y}_0^{\text{EM}}|^2] +   16 \mathsf{K}^2_{\text{Total}} T (1+  2  \widetilde{\varepsilon}_{\text{AL}} + 2 | \theta^{*}|^2) (1+T^{2 \alpha })  +  2d T  )
			\end{aligned}
			$
			\\\hline
			$C_{\mathsf{EM},4}(T)$ & $O (d^2 e^{T^{4\alpha+1}}  T^{4 \alpha +1}) $ & &
			$
			\begin{aligned}
				& e^{T(  \frac{21}{2}   + 128 \mathsf{K}^4_{\text{Total}} (1+T^{4 \alpha }))}
				\\ & \quad \times   ( \mathbb{E} [ |\widehat{Y}_0^{\text{EM}}|^4] +  1024 \mathsf{K}^4_{\text{Total}} T (1+ \mathbb{E} [  |  \hat{\theta} |^4  ]) (1+T^{4 \alpha })  + 8(d^2+ 4d +4) T  )
			\end{aligned}
			$
			\\\hline
			$	C_{\mathsf{EMose},2} $ & $O (d e^{T^{2\alpha+1}}  T^{4 \alpha +1} \widetilde{\varepsilon}_{\text{AL}})  $ & &
			$
			\begin{aligned}
				2 (C_{\mathsf{EM},2}(T) + \mathsf{K}^2_{\text{Total}} (1+  T^{2 \alpha } )  (16    C_{\mathsf{EM},2}(T) + 32 (1+  2  \widetilde{\varepsilon}_{\text{AL}} + 2 | \theta^{*}|^2)  )) +   2 d
			\end{aligned}
			$
			\\ \hline
			$C_4(T,\epsilon)$ & $O(\sqrt{d} e^{T^{2\alpha +1}} T^{3 \alpha +1} \widetilde{\varepsilon}^{1/2}_{\text{AL}}) $ &  &
			\begin{math}
				\begin{aligned}  
					&\zeta^{-1/2} (T-\epsilon)^{1/2}  e^{(1+ (3/2)\zeta + 2  \mathsf{K}_3 (1+2T^{\alpha})  )(T-\epsilon)} 
					\\& \qquad \times     (   C^{1/2}_{\mathsf{EMose},2}(1+ 2  \mathsf{K}_3 (1+2T^{\alpha})    ) + 2 \sqrt{2}     \mathsf{K}_1  (1+8  \widetilde{\varepsilon}_{\text{AL}} + 8 | \theta^{*}|^2)^{1/2})
				\end{aligned} 
			\end{math}
			\\\hline
			$	C_{\mathsf{EMose},4} $ & $O (d^2 e^{T^{4\alpha+1}}  T^{8 \alpha +1})  $ & &
			$
			\begin{aligned}
				8 (C_{\mathsf{EM},4}(T) + \mathsf{K}^4_{\text{Total}} (1+  T^{4 \alpha } )  (1024  C_{\mathsf{EM},4}(T) + 8192 (1+ \E [|\hat{\theta}|^4])  )) + 144   d^{2}
			\end{aligned}
			$
			\\\hline
			$\widetilde{C}_4(T,\epsilon)$ & $O(d  e^{T^{4 \alpha+1}} T^{4 \alpha +1} \widetilde{\varepsilon}^{1/4}_{\text{AL}}) $& &   	$
			\begin{aligned} & \sqrt{2} e^{ 2(1+  \zeta +  \mathsf{K}_3  (1+2 T^{\alpha} + 4  \mathsf{K}_3  (1+ 4T^{2 \alpha}) )) (T-\epsilon) }  \sqrt{T-\epsilon}
				\\ & \quad \  \times    \Bigg(    \mathsf{K}_4^2 \zeta^{-1} (1+4 T^{2\alpha})  C_{\mathsf{EMose},4}  +  4  d (1 +8 \mathsf{K}^2_3 (1+ 4  T^{2\alpha} ) ) \\ & \qquad \quad +    2 \zeta^{-1} \mathsf{K}_1^2 (1+ 8(   \widetilde{\varepsilon}_{\text{AL}} +  | \theta^{*}|^2) )
				\\ & \qquad \quad + 4    \zeta^{-1}  d   (1 +8 \mathsf{K}^2_3 (1+ 4  T^{2\alpha} ) )
				\\ &  \qquad \quad  \quad \times [ (1+ 16 \mathsf{K}^2_{\text{Total}} (1+T^{2 \alpha})) C_{\mathsf{EM},2}(T) \\ &  \qquad \qquad  \qquad + 32 \mathsf{K}^2_{\text{Total}}  (1+T^{2 \alpha})(1+  2  \widetilde{\varepsilon}_{\text{AL}} + 2 | \theta^{*}|^2) ]
				\\ &  \qquad \quad   + 2  [    (1+ 8 \mathsf{K}_3^2(1+ 4T^{2 \alpha}))^{1/2} C^{1/2}_{\mathsf{EMose},2} +   2 \mathsf{K}_1       (1+ 8  \widetilde{\varepsilon}_{\text{AL}} + 8 | \theta^{*}|^2)^{1/2}]
				\\ &  \qquad \qquad \times  [  d \sqrt{2} (  1+ 8  \mathsf{K}_3^2 (1+ 4  T^{2\alpha} )            )^{1/2}] \Bigg)^{1/2} 
			\end{aligned}
			$  \\ 
			\\ \hline	$\epsilon_{\delta}$& - & & $\delta^2/(64( \sqrt{\mathbb{E}[|X_0|^2]}+\sqrt{d}) ^2)$ 		\\ \hline
			$T_{\delta}$& - & &  \makecell{Obtained solving $T>T_{\delta}$ using Proposition \ref{Time_integral_contraction_constant_proposition}, i.e., \\ 
				\begin{math}
					\begin{aligned}  
						&	\ln(\mu(e^{2T}-1)+1) +(K/\mu +1)/(\mu(e^{2T}-1)+1) \\ & > \ln(4\sqrt{2}((\E[|X_0|^2])^{1/2} + \sqrt{d})/\delta) + 2 \int_0^{\epsilon} \beta^{\text{OS}, K, \mu}_t   \ \rmd t  + K/\mu +1 - \epsilon 
					\end{aligned} 
				\end{math} 
			}
			\\ \hline
			$\varepsilon_{\text{SN}, \delta}$ & - & &   	\begin{math}
				\begin{aligned}  
					(\delta^2 \zeta /32) e^{-2(1+\zeta)(T-\epsilon)+4 \int_{\epsilon}^T \beta^{\text{OS}, K, \mu}_t   \ \rmd t}
				\end{aligned} 
			\end{math} 
			\\ \hline $ \gamma_{\delta}$ & - & &  	\begin{math}
				\begin{aligned}   
					& (\delta^2 \zeta/16) (T-\epsilon)^{-1}  e^{-2(1+ (3/2)\zeta + 2  \mathsf{K}_3 (1+2T^{\alpha})  )(T-\epsilon)} 
					\\& \quad \times     (   C^{1/2}_{\mathsf{EMose},2}(1+ 2  \mathsf{K}_3 (1+2T^{\alpha})    ) + 2 \sqrt{2}     \mathsf{K}_1  (1+8  \widetilde{\varepsilon}_{\text{AL}} + 8 | \theta^{*}|^2)^{1/2})^{-2}
				\end{aligned} 
			\end{math} 
			\\ \hline
			$\widetilde{\gamma}_{\delta}$&- & &
			$
			\begin{aligned}
				\min\Bigg\{ &  (\delta / (4 \sqrt{2}))^{1/\alpha}  (T-\epsilon)^{- 1/(2\alpha)} e^{ -(2/\alpha)(1+  \zeta +  \mathsf{K}_3  (1+2 T^{\alpha} + 4 \mathsf{K}_3  (1+ 4T^{2 \alpha}) ))(T-\epsilon)}
				\\ & \times  \Bigg(  \mathsf{K}_4^2 \zeta^{-1} (1+4 T^{2\alpha})  C_{\mathsf{EMose},4}  + 4  d (1 +8 \mathsf{K}^2_3 (1+ 4  T^{2\alpha} ) ) \\ & \qquad  +    2 \zeta^{-1} \mathsf{K}_1^2 (1+ 8(   \widetilde{\varepsilon}_{\text{AL}} +  | \theta^{*}|^2) )
				\\ & \qquad  + 4    \zeta^{-1}  d   (1 +8 \mathsf{K}^2_3 (1+ 4  T^{2\alpha} ) )
				\\ &  \qquad \quad   \times [ (1+ 16 \mathsf{K}^2_{\text{Total}} (1+T^{2 \alpha})) C_{\mathsf{EM},2}(T) \\ & \qquad \qquad  + 32 \mathsf{K}^2_{\text{Total}}  (1+T^{2 \alpha})(1+  2  \widetilde{\varepsilon}_{\text{AL}} + 2 | \theta^{*}|^2) ]
				\\ &  \qquad    + 2  [    (1+ 8 \mathsf{K}_3^2(1+ 4T^{2 \alpha}))^{1/2} C^{1/2}_{\mathsf{EMose},2} +   2 \mathsf{K}_1       (1+ 8  \widetilde{\varepsilon}_{\text{AL}} + 8 | \theta^{*}|^2)^{1/2}]
				\\ &  \qquad \qquad \times  [  d \sqrt{2} (  1+ 8  \mathsf{K}_3^2 (1+ 4  T^{2\alpha} )            )^{1/2}] \Bigg)^{-1/(2\alpha)} ,  1\Bigg\}
			\end{aligned}
			$
			
			\\
			\bottomrule
		\end{tabular}
		\label{tab:convconst_general}
		
		
	\end{table}	
\newpage
Table \ref{Table_notation}  provides the main notation used throughout this work and indicates where each symbol is introduced.
\begin{table}[h]
\centering
	\caption{List of the main notation.} 
	\begin{tabular}{ |p{3cm} | p{6cm}| } 
\hline
\begin{sc}
Symbol
\end{sc} &\
\begin{sc} Reference in the text \end{sc}  \\
\hline
 $\pi_{\mathsf{D}}$  & Equation \ref{OU_process_introduction} \\
$\hat{\theta}$ & Assumption~\ref{general_assumption_algorithm} \\
$\widetilde{\varepsilon}_{\text{AL}}$ & Assumption~\ref{general_assumption_algorithm}
\\
$\partial U$ & Definition \ref{definition_subdifferential}
\\ 
$K$ & Assumption~\ref{assumption_score_strong_monotonicity_new_semi_convexity}-(ii) \\
$\mu$ & Assumption~\ref{assumption_score_strong_monotonicity_new_semi_convexity}-(iii)
\\ $\alpha$ & Assumption \ref{Assumption_2_without_derivative}
\\ 
$\varepsilon_{\text{SN}}$  & Assumption \ref{assumption_equivalence_global_minimiser_epsilon}
\\
$\kappa_U$  & Definition~\ref{definition_weak_convexity_profile_new} \\
$f_L$ & Equation~\ref{definition_f_L_assumption_new} \\
$\beta^{\text{OS}, K, \mu}_t $ & Equation \ref{weak_convexity_contraction_constant_K_mu}
\\
$B(t,0,  \mu, K)$ & Equation \ref{expression_constant_A_t} \\
		\hline
	\end{tabular}
	\label{Table_notation}
\end{table}
\end{document}